%% file: neurips_2026.tex
\documentclass{article}

\usepackage[main, final]{neurips_2026}

\usepackage[utf8]{inputenc} 
\usepackage[T1]{fontenc}    
\usepackage{hyperref}       
\usepackage{url}            
\usepackage{booktabs}       
\usepackage{amsfonts}       
\usepackage{nicefrac}       
\usepackage{xcolor}         

\usepackage{amsmath}
\usepackage{amssymb}
\usepackage{mathtools}
\usepackage{amsthm}
\usepackage{wrapfig}
\usepackage{titletoc}

\usepackage[capitalize,noabbrev]{cleveref}
\hypersetup{hidelinks}

\author{%
  Joseph Wilson \\
  School of Mathematics and Physics\\
  University of Queensland\\
  \texttt{joseph.wilson1@uqconnect.edu.au} \\
  \And
  Chris van der Heide \\
  Dept. of Electrical and Electronic Engineering\\
  University of Melbourne \\
  \texttt{chris.vdh@gmail.com} \\
  \And
  Liam Hodgkinson \\
  School of Mathematics and Statistics\\
  University of Melbourne \\
  \texttt{lhodgkinson@unimelb.edu.au} \\
  \And
  Zhichao Wang \\
  ICSI and Department of Statistics\\
  University of California, Berkeley\\
  \texttt{zhichao.wang@berkeley.edu} \\
  \And
  Fred Roosta \\
  CIRES and School of Mathematics and Physics\\
  University of Queensland\\
  \texttt{fred.roosta@uq.edu.au} \\
  \And
  Michael W. Mahoney \\
  ICSI, LBNL and Department of Statistics\\
  University of California, Berkeley\\
  \texttt{mmahoney@stat.berkeley.edu} 
}

\input{Format/packages}
\input{Format/newcmnds}

\title{
Why Heavy-Tailed Weight Spectra \\
Predict Model Quality
}
\begin{document}
\maketitle

\begin{abstract}
A solid understanding of the predictive success of
deep neural networks (DNNs) remains elusive. 
For example, while many data-dependent metrics exist that seek to predict model quality of DNNs, these metrics may be expensive to compute for large datasets and/or may be impossible to compute when training datasets are not released. 
A practical and effective predictor of model quality involves analyzing DNN weight matrices, as prior work has established that heavy-tailed spectral distributions in trained weight matrices can be strongly predictive of model quality. 
However, such metrics lack a rigorous statistical justification. 
In this work, we aim to understand why these metrics perform well, from a statistical perspective. 
To do so, we employ a probabilistic framework to derive the marginal likelihood for a DNN, evaluated using only the final layer weights.
Across large-scale convolutional DNNs and LLMs, we find that this marginal likelihood predicts model quality. 
To understand the role heavy-tailed spectra play in 
predictive performance,
we derive the limiting log-marginal likelihood for spectra that follow the (heavy-tailed) High-Temperature Marchenko-Pastur (HTMP) distribution. 
We show that this limiting quantity is convex in a \textit{heaviness} parameter, and we derive an optimal \textit{heaviness} that empirically predicts over-fitting for large-scale DNNs. 
\end{abstract}

\section{Introduction}
Since the introduction of AlexNet \citep{krizhevsky2012imagenet}, an understanding of the mechanisms behind 
the predictive success 
of deep neural networks (DNNs) has lagged behind empirical advancements \citep{zhang2021understanding}. 
For example, generalization is one way to try to quantify this; in contrast to classical statistical learning theory (SLT), however, large heavily over-parameterized DNNs generalize well after training to near zero-loss.
This phenomenon extends the classical U-shaped bias-variance curve to the well-known \textit{double-descent} curve \citep{belkin2019reconciling}. 
The recent success of large language models (LLMs) further complicates the picture: foundation models are often \textit{under-parameterized}, and do not attain zero-loss \citep{hoffmann2022training}. 
Instead, the performance of LLMs seems to be (empirically) well-governed by neural scaling laws \citep{kaplan2020scaling, hoffmann2022training}, the theoretical underpinnings of which are not well understood \citep{boopathy2024unified}. 

Tools for attempting to estimate model quality of DNNs are extensive. 
Classical statistical techniques using complexity measures such as the Rademacher complexity, Vapnik–Chervonenkis (VC) dimension, and uniform stability have been largely unsuccessful \citep{zhang2016understanding}. 
The celebrated PAC-Bayes bounds \citep{mcallester2003simplified} provide improved generalization bounds, yet do not provide a sufficient explanation of generalization \citep{lotfi2022pac, livni2020limitation}. 
A related quantity is the \emph{marginal likelihood} of a DNN, which is connected to both PAC-Bayes bounds \citep{germain2016pac} and leave-$k$-out cross-validation \citep{fong2020marginal}. \citet{immer2021scalable} demonstrated correlation with model quality for large DNNs, although certain caveats have been reported \citep{lotfi2022bayesian}. 
This quantity is important in SLT \citep{konishi2008information}, since from it the Bayesian Information Criterion and Watanabe–Akaike Information Criterion can be derived. 
For singular models (e.g., DNNs), the Widely Applicable Bayesian Information Criterion \citep{watanabe2013widely} was developed; it can predict quality of pre-training for downstream tasks \citep{munn2024bayesian}. 
These information criteria are derived in the regime where number of data-points $n \to \infty$, for fixed parameter dimension $p$; but this is no longer a valid assumption for modern deep learning \citep{couillet2022random}. 
To account for this, the Interpolating Information Criterion (IIC) \citep{hodgkinson2023interpolating, hodgkinson2023pac} 
was developed for models that attain zero-loss, although we should note that the IIC depends on an implicit regularizer for DNNs, whose full characterization remains a challenging open problem. 

\begin{figure}[t]
    \centering
    \begin{subfigure}[b]{0.32\textwidth}
        \includegraphics[width=\textwidth]{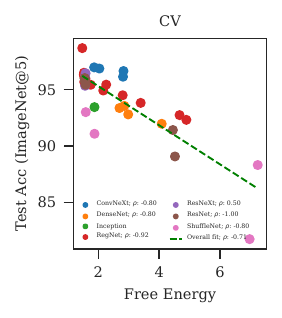}
    \end{subfigure}
    \begin{subfigure}[b]{0.32\textwidth}
        \includegraphics[width=\textwidth]{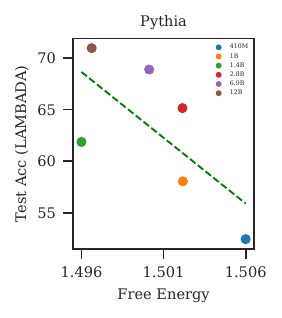}
    \end{subfigure}
    \begin{subfigure}[b]{0.32\textwidth}
        \includegraphics[width=\textwidth]{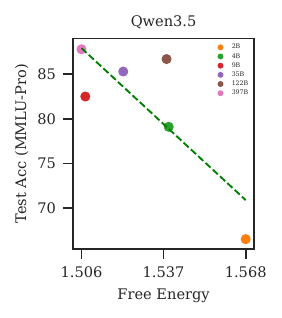}
    \end{subfigure}
    \caption{Plots of \cref{eq:bfe_weights} versus test accuracy, for (left) \textit{torchvision} models trained on ImageNet-1K, (middle) Pythia models, and (right) Qwen3.5 models. Test dataset is given in brackets.}
    \label{fig:free_energy_vs_test_acc}
    \vspace{-2mm}
\end{figure}

These techniques for predicting model quality are \emph{data-dependent}, which presents multiple issues for their practical use. 
Firstly, computation on massive training sets can become infeasible \citep{AMELI-2025b,ameli2025determinant}. 
Secondly, foundation models are often trained on undisclosed or proprietary datasets \citep{singh2025openai}, limiting the ability for comparisons across frontier labs. 
In response, a novel line of work has considered using DNN weight matrices in isolation (\emph{data-independent}, i.e., without training or test data) for predicting model quality \citep{martin2020heavy, yang2022evaluating}. 
Motivation arose from the Heavy-Tailed Mechanistic Universality (HT-MU) \citep{MM19_HTSR_ICML,martin2020heavy} and Heavy-Tailed Self-Regularization (HT-SR) theories \citep{martin2021implicit}.
These related approaches are phenomenological, or semi-empirical, theories that are motivated by statistical mechanics \citep{martin2017rethinking,martin2021implicit}, and they are based on the empirical observation that \emph{spectra} of trained DNN weight matrices exhibit \emph{heavy-tails}. 
Metrics derived from the ``heaviness'' of the spectral tail are surprisingly strong predictors of model quality for both computer vision \citep{martin2021predicting} and language models \citep{yang2022evaluating}. 
Such metrics can identify DNN layers that are ``under-trained,'' allowing for layer-wise learning \citep{zhou2023temperature} or pruning \citep{lu2024alphapruning}. 
Motivated by this, \citet{hodgkinson2025models} recently developed the High-Temperature Marchenko-Pastur (HTMP) spectral distribution, which (among other things) is appropriate for modeling (heavy-tailed spectra of) weight matrices \textit{after} training. 
This probabilistic model exhibits \emph{heavy-tailed spectra} for weight matrices that do \textit{not} necessarily possess heavy-tailed \textit{elements}, a phenomenon observed in practice \citep{martin2021implicit}, and it can reproduce neural scaling laws for generalization error.

In this work, we aim 
to understand---from a statistical perspective---\emph{why} heavy-tailed weight spectra predict model quality.
To do so, we reframe and derive certain aspects of HT-SR for DNNs using the tools of SLT and Random Matrix Theory (RMT). 
Doing so will in turn allow DNN theory to more readily use and build upon the HT-SR phenomenology. 
To this end, we employ a stochastic modeling framework for DNN training, and a weights and features (i.e., final activations) decomposition, to derive the marginal likelihood for \textit{trained features} $\sZ(\Phi)$. We then show that $\sZ(\Phi)$ can be evaluated using only \textit{final-layer weights}. Defining the Bayes free energy as $\sF := -\log \sZ$, we observe the following behavior:
\begin{tcolorbox}
    For the \textit{trained} weights $\WW$, $m$ outputs, and correctly chosen regularization parameter $\lambda$,
    \begin{align*}
        \sF_{m}^{\lambda}(\WW) = \frac{\lambda}{2 m} \Tr (\WW^T \WW) - \frac{1}{2m} \log \det \WW^T \WW + \frac{1}{2} \log \frac{2 \pi}{\lambda}
    \end{align*}
    is strongly predictive of 
    model quality for large-scale computer vision and language models.
\end{tcolorbox}

We illustrate the predictive quality of this metric for many large models in \cref{fig:free_energy_vs_test_acc}. 
As this metric only requires the spectrum of $\WW$, i.e., \emph{it does not require access to training or testing data}, we can then explore a central question from HT-SR: \emph{why} does the heaviness in the spectral tail of weight matrices predict model quality? 
To answer this, we employ the HTMP distribution, with heaviness parameters $\kappa, \beta$, to derive the deterministic limit of $\sF_{m}^{\lambda}(\WW)$, 
under the double-asymptotic scaling of RMT. From this quantity, we find two key spectral parameters predict model quality:

\begin{tcolorbox}
    \begin{theorem_nonumber}[Informal]
        Under the HTMP, for fixed $\kappa$, there exists an optimal ratio $\kappa / \beta$; once this optimal ratio is achieved, $\sF_{m}^{\lambda}(\WW)$ is strictly decreasing in $\kappa$.
     \end{theorem_nonumber}
\end{tcolorbox}

In a series of experiments, we observe that for large DNNs, $\sF_{m}^{\lambda}(\WW)$ correlates with test error very well \textit{during training}. 
In addition, for \emph{trained} Computer Vision (CV) models, the ratio $\kappa / \beta$ controls underfitting / overfitting. 
In contrast, \emph{trained} LLMs attain the optimal ratio early; and increasing $\kappa$ during training then correlates with improved performance. 
These results provide a probabilistic connection between spectral properties and model quality.


\begin{figure}[t]
    \centering
    \begin{subfigure}[b]{1\textwidth}
        \includegraphics[width=\textwidth]{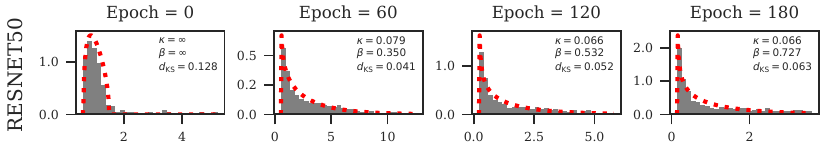}
    \end{subfigure}
    \begin{subfigure}[b]{1\textwidth}
        \includegraphics[width=\textwidth]{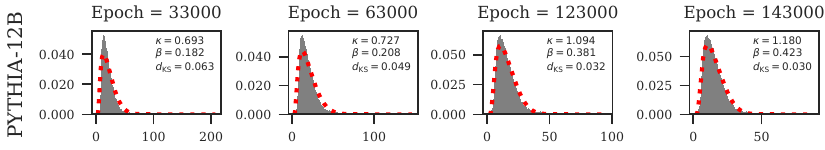}
    \end{subfigure}
    \caption{The ESD (gray) and fitted $\textbf{HTMP}_{\kappa, \gamma, \beta, \tau}$ distribution (red dotted lines), for the final-layer weights of (top) ResNet50 trained on CIFAR-100, and (bottom) Pythia-$12B$ trained on Pile.}
    \label{fig:esd_training}
    \vspace{-4mm}
\end{figure}

\paragraph{Contributions.} 
To summarize, in this work, we do the following:
\begin{enumerate}
    \item Derive the equivalent \textit{trained} last-layer weights formulation of the Bayes free energy of a DNN.
    \item Empirically demonstrate strong correlation of this free energy with model quality.
    \item Derive the limiting free energy under HTMP, and find the minimizing HTMP parameters.
    \item Empirically show that training of DNNs corresponds to optimization of HTMP parameters.
\end{enumerate}

\cref{sec:background} provides necessary background. 
In \cref{sec:free_energy}, we derive $\sF_{m}^{\lambda}(\WW)$ and provide an empirical evaluation. 
In \cref{sec:heavy_tail}, we find the limit of this quantity under the HTMP distribution, and we show how this quantity can be optimized; this is then connected to the training of practical DNNs. 
Finally, a conclusion and discussion is given in \cref{sec:conclusion}. 
In the Appendix, we give proofs for the theoretical results, experimental details, further discussions, and further experimental and theoretical results.
Code to reproduce the experiments can be accessed~\href{https://github.com/josephwilsonmaths/heavytail.git}{here}.

\section{Related work \& background}
\label{sec:background}

\paragraph{RMT in deep learning.} 
RMT broadly studies the limiting spectrum of matrix-valued probabilistic models. There is a long line of works that derive deterministic equivalents and generalization error for linear or random DNNs using RMT \citep{liao2020random, mei2022generalization, wang2024nonlinear, atanasov2025two}.  For a two-layer DNN, the generalization error after one gradient step was derived in \citep{ba2022high}. The performance of a fine-tuned DNN linearization was predicted through a Gaussian-like local spectral law \citep{wei2022more}.
The limiting free energy was derived in \citep{hodgkinson2023monotonicity} for Gaussian Processes (GP) with inner-product or radial-basis kernels, and in \citet{wilson2026last} for GPs with kernels derived from DNNs, using spectra derived from \citep{fan2020spectra}.
RMT for deep learning has recently been reviewed \citep{LiaoMahoney2026}. In RMT, one considers the \textit{double asymptotic limit}, i.e., for matrix $\AA \in \real^{m \times n}$, take $m, n \to \infty$ such that $m / n \to \gamma \in (0,\infty)$.

\paragraph{Implicit regularization.}
While \emph{explicit} forms of regularization for DNNs are often employed, certain aspects of a DNN such as architecture choice (e.g., depth) or optimization scheme (small-batch SGD) may \emph{implicitly} \citep{MO11-implementing,Mah12} bias a DNN towards parameter values that generalize better. For homogeneous DNNs on linearly separable data, steepest descent with respect to any norm $\|\cdot\|$ will implicitly minimize the parameter norm past the point of interpolation \citep{tsilivis2026flavors}. Small-batch SGD is thought to bias DNNs towards \textit{flatter minima} during training, improving generalization \citep{keskar2016large}, although this idea has received criticism \citep{dinh2017sharp,YGKM19_pyhessian_TR,YGSKM20_adahessian_TR}. The Geometric Complexity (GC) \citep{dherin2022neural}, defined by a discrete version of the Dirichlet Energy \citep{jost2017riemannian}, is a measurement of geometric smoothness that encompasses many previous notions of smoothness and seems to be correlated with model quality. A related object \citep{munn2024bayesian} is the Real Log Canonical Threshold from the Singular Learning Theory of \citet{watanabe2009algebraic}, and its local approximation, the Local Learning Coefficient \citep{lau2023local} which also seems to correlate with model quality. We emphasize that these implicit regularization metrics depend on training data.

\paragraph{Heavy-tailed self-regularization (HT-SR).}
HT-SR \citep{martin2021implicit} posits that the empirical spectral distributions (ESD) of trained weight matrices follow a power-law (PL) ($\rho(u) \propto u^{-\alpha},u > u_{\min}, \alpha> 1$), or exponentially-truncated PL (ET-PL) ($\rho(u) \propto u^{-\alpha} e^{-\lambda u},u > u_{\min}, \alpha > 1, \lambda > 0$) in the tail, due to complex correlations learnt during training, for some chosen minimum eigenvalue $u_{\min}$. 
Using the fitted $\alpha$ value, in \citet{martin2021predicting}, it was shown that for trained matrices $\WW$, where $\spec \left(\WW^T \WW\right) = \left\{u_i\right\}_{i=1}^{d_l}$, the metrics $\|\WW\|_F^2 = \sum u_i$, $\|\WW\|_{\infty}^2 = u_{\max}$, $\hat{\alpha} = \alpha \log u_{\max}$ and $\|\WW\|_{2\alpha}^{2\alpha} = \sum u_i^{\alpha}$, averaged over weights in a network, were strongly predictive of model-quality for CV tasks, when comparisons are made \textit{within-group}. 
HT-SR can be heuristically derived based on ideas from the statistical mechanics of learning~\citep{martin2017rethinking,martin2021implicit,setolTR}.

\paragraph{Entropic regularization framework.}
Define the stochastic minimization operation, $\smin_{\btheta}^{\pi_{\btheta}, \tau}$, with temperature $\tau$ and prior distribution $\pi_{\btheta}$, for $\btheta \in \real^p$, and $\sL : \real^p \to \real$, as
\begin{align}\label{eq:smin}
    \smin_{\btheta}^{\pi_{\btheta}, \tau} \sL(\btheta) := \min_{q \in \sP} \left[ \Ex_{q(\btheta)} \sL(\btheta) + \tau d_{\text{KL}}(q || \pi_{\btheta}) \right].
\end{align}
Here, $\sP$ is the set of probability distributions on the support of $\pi_{\btheta}$. We denote the minimizing density over parameters as $q(\theta) = \argsmin_{\btheta}^{\pi_{\btheta}, \tau} \sL(\btheta)$, provided it exists and is unique. Such probabilistic frameworks of training have been considered previously \citep{xie2022power, sonthalia2024regularization}.
Note that \cref{eq:smin} optimizes a PAC-Bayes bound on the test-error when $L$ is the training loss \citep[Appendix B.]{hodgkinson2025models}, and it is related to the variational inference performed during stochastic gradient descent \citep{chaudhari2018stochastic}.



\section{Free energy for trained weights}
\label{sec:free_energy}
In this section, we first derive $\sF_{m}^{\lambda}(\WW)$, the negative log-marginal for interpolating \textit{trained features}, that can be evaluated using only trained \textit{final-layer weights}. We then show that this quantity correlates very well with model quality for many large DNNs.

\paragraph{Notation.}
Consider a DNN $\ff : \real^{d} \times \real^p \to \real^m$, for inputs $\xx \in \sX \subseteq \real^d$ and parameters $\btheta \in \bm\Theta \subseteq \real^p$. 
We take training dataset $\sD =  \left \{ \xx_i, \yy_i\right \}_{i=1}^n \subseteq \sX \times \sY$, for $\sY \subseteq \real^m$, and loss function $\sL(\btheta)$. 
We denote the concatenations of training data as $\XX = [\xx_1 | \dots | \xx_n]^T \in \real^{n \times d}$, $\YY = [\yy_1 | \dots | \yy_n]^T \in \real^{n \times m}$. 
In our setting, it will be useful to consider a decomposition of a DNN into its \textbf{non-linear features} and \textbf{final-layer weights}. We denote the function space of features as $\sF$; then we write $\ff : \real^{d} \times \real^{d_L \times m} \times \sF \to \real^m$, for inputs $\xx \in \sX \subseteq \real^d$, final-layer weights $\WW \in \real^{d_L \times m}$, and feature function $\phi \in \sF$. The functional form of the DNN is $\ff(\xx ; \WW, \phi) = \phi(\xx)^T \WW$. We can represent the function evaluated on training inputs $\XX$ as $\ff(\XX ; \WW, \phi) = \Phi \WW$, where $\Phi$ is the concatenation $\Phi = [\phi(\xx_1) | \dots | \phi(\xx_n)]^T \in \real^{n \times d_L}$. We denote the matrix-normal distribution by $\sM \sN ( ., .,.)$.

\vspace{-2mm}
\subsection{Optimal features under interpolation for ridge regression}\label{sec:ridge_regression}
The exact training dynamics of large DNNs are highly complex, prohibiting careful analysis. To simplify analysis, we consider training $\ff$ on $\sD$, using the \textbf{entropic regularization} framework in \cref{eq:smin}, for loss function $L(\WW, \Phi) = ||\Phi \WW - \YY||_F^2$, and $\pi_{\WW} = \sM \sN ( \bm 0, \sigma^2\eye_{d_L},\eye_m)$. Under this framework, optimal \textbf{feature} density $q(\Phi)$ is
\begin{align*}
    q(\Phi) = \argsmin_{\Phi}^{\pi_{\Phi}, \eta} \smin_{\WW}^{\pi_{\WW}, \tau} \sL(\WW, \Phi),
\end{align*}
provided the minimum exists and is unique. By \citet[Proposition 3.1]{hodgkinson2025models}, $q(\Phi) \propto \sZ_{\tau}(\Phi)^{\tau / \eta}\pi_{\Phi}(\Phi)$, where $\sZ_{\tau}(\Phi)$ is the marginal (Gibbs) likelihood with prior $\pi_{\WW}$, $\sZ_{\tau}(\Phi) \propto \Ex_{\pi_{\WW}} \exp(-L(\WW, \Phi) / \tau)$. 
So, we observe that trained features concentrate around regions with high marginal likelihood. 
Under our loss function and prior, we can explicitly compute (calculations are given in \cref{sec:appendix:marginal}) the negative log marginal for the trained features (i.e., Bayes Free Energy):
\begin{align}
    \frac{1}{nm} \sF^{\tau_m, \lambda_m}_n(\Phi) &:= \frac{\lambda_m}{2n} \Tr \left( \YY^T \KK_{\lambda_m, \tau_m}^{-1} \YY \right) + \frac{1}{2n}\log \det \KK_{\lambda_m, \tau_m} + \frac{1}{2}\log \left(\frac{2\pi}{\lambda_m m} \right), \label{eq:bfe_phi}
\end{align}
where $\KK_{\lambda_m, \tau_m} = \Phi \Phi^T + \lambda_m \tau_m / 2\eye_n$, $\lambda_m = 1 / (m\sigma^2)$ and $\tau_m = \tau m$. 

\subsection{Evaluation of free energy using final-layer weights}
For \textbf{trained} features $\Phi$, the Bayes Free Energy can be used to measure model quality; while this is not novel, the following extension is: under certain assumptions, for training temperature $\tau \to 0$, the Bayes Free Energy for well-trained features can be evaluated using \textbf{only} the final-layer weights $\WW$, up to some additive terms that cancel when comparing Bayes Free Energy between models trained on the same dataset. 
Taking $\tau \to 0$ concentrates the marginal likelihood on the sub-manifold of interpolators; this concentration approach was used to derive the IIC \citep{hodgkinson2023interpolating}. 
\begin{assumption}[Interpolation]\label{ass:features}
    For \textbf{trained} features $\Phi$, assume $n \geq m, d_L > m$ and $\rank(\Phi) = \rank(\YY) = m$. Further, assume that \textbf{trained} weights are given by $\WW = \Phi^{\dagger} \YY$, and that $\Phi \WW = \YY$.
\end{assumption}

This minimum-norm interpolating solution $\WW = \Phi^{\dagger} \YY$ is key; $\WW$ contains all necessary information about the relationship between features and targets. These assumptions have strong motivations: under our entropic regularization training framework, taking $\tau,\sigma^2 \to 0$ such that $\tau/\sigma^2 \to 0$ ensures $\WW \to \Phi^{\dagger} \YY$, and the same result applies in the deterministic optimization setting \citep{golub1973differentiation}. 
Further, in a sufficiently over-parameterized multiclass logistic regression setting, \citet{wang2021benign} show trained weights ``interpolate'' the centered training outputs. 
Under neural collapse assumptions \citep{papyan2020prevalence}, these results extend to over-parameterized neural networks. 
Using this interpolating relationship, we can now evaluate \cref{eq:bfe_phi} using only the final-layer weights.

\begin{theorem}\label{thm:bfe_w}
    Define $\Sigma := \YY^T \YY$. Then, under \cref{ass:features}, as $\tau \to 0^+$,
    \begin{align*}
        \frac{\mathcal{F}^{\lambda_m, \tau_m}_n (\Phi)}{nm}  
        &= \frac{\Tr (\lambda_m\WW^T \WW)}{2n}  - \frac{\log \det \left(\lambda_m\WW^T \WW\right)}{2n}  + \frac{\log(\pi\tau)^{n-m}}{2n} +\frac{\log \det 2 \pi \Sigma/m}{2n}  + \mathcal{O}(\tau). 
\end{align*}
\end{theorem}

The proof of this theorem is given in \cref{sec:appendix:weight_energy}. 
We observe that the final three terms will cancel as $\tau \to 0$ when comparing models trained on the same dataset. Dropping the dependence of $\lambda$ on m, dropping the constant terms that cancel (except for a Gaussian entropy term), and correctly normalizing the trace and log-determinant terms, we arrive at the \textbf{relative free energy}:
\begin{tcolorbox}
\vspace{-3mm}
    \begin{align}\label{eq:bfe_weights}
    \mathcal{F}_{m}^{\lambda}(\WW) := \frac{\lambda}{2 m} \Tr (\WW^T \WW) - \frac{1}{2m} \log \det \WW^T \WW + \frac{1}{2} \log \frac{2 \pi}{\lambda}. \tag{BFE-$\WW$}
\end{align}
\end{tcolorbox}

Thus, we have derived, under an entropic regularization framework, a data-free metric that directly measures model quality of the underlying DNN features. We see that \cref{eq:bfe_weights} involves a trade-off between $m^{-1}\Tr (\WW^T \WW)$ and $m^{-1}\log \det \WW^T \WW$, controlled by parameter $\lambda$ (prior scale). The advantage of \cref{eq:bfe_weights} over \cref{eq:bfe_phi} for comparing models is immediate: \cref{eq:bfe_weights} does \emph{not} require access to training or validation data, and $\WW$ is often of a \emph{much} smaller size than $\Phi$. 

Note that when $m > d_L$, $\WW^T \WW$ is rank-deficient. 
In practice, one might add a small ridge-parameter $\epsilon$ for \cref{eq:bfe_weights}, i.e., $\WW^T \WW + \epsilon \eye$. 
However, one can also derive a dual formulation of \cref{eq:bfe_weights}, to obtain:
\begin{align*}
    \tilde{\mathcal{F}}_{d_L}^{\lambda}(\WW) := \frac{\lambda}{2 d_L} \Tr (\WW \WW^T) - \frac{1}{2d_L} \log \det \WW \WW^T + \frac{1}{2} \log \frac{2 \pi}{\lambda}. 
\end{align*}
The full asymptotic expansion is given along with its derivation in \cref{lemma:bfe_w_dual} in \cref{sec:appendix:dual_derive}. While this free energy does not immediately allow for comparison between models with differing $d_L$, we will see that it still correlates surprisingly well with model quality, across models with differing $d_L$.

\begin{remark}\label{remark:lambda}
To ensure \cref{eq:bfe_weights} correlates with the posterior predictive loss, one can set $\lambda^* = \argmin_{\lambda} \mathcal{F}_{m}^{\lambda}(\WW)$ \citep{hodgkinson2023interpolating, hodgkinson2023monotonicity, lotfi2022bayesian, mackay1992bayesian}. Taking a derivative immediately gives $\lambda^* = m / \Tr(\WW^T \WW)$. 
Interestingly, at this scale, 
\begin{align*}
    \mathcal{F}_{m}^{\lambda^*}(\WW) = \frac{1 + \log 2 \pi}{2} + \frac{1}{2}\log \left( \frac{m^{-1} \Tr ( \WW^T \WW)}{\det(\WW^T \WW)^{1/m}}\right).
\end{align*}
Hence, at the \textit{optimal regularizer}, the free energy is measuring the log-ratio of the arithmetic spectral mean versus geometric spectral mean. 
This is trivially minimized at $\WW \propto \UU$ (up to a scalar multiple), where $\UU$ has orthonormal columns; this corresponds to the case when $\Phi \Phi^T \propto \YY \YY^T$. 
Outside of this ``optimal-learning scenario,'' behavior is non-trivial.
\end{remark}

\subsection{Empirical evaluation}
\label{sec:free_energy_experiment}
Here, we empirically demonstrate the correlation of \cref{eq:bfe_weights} with model quality.

\paragraph{Computer vision (CV).} 
We take a large number of high-performing models, across different architecture types (e.g., ResNet, DenseNet etc.) from the \textit{torchvision} model zoo that have been pretrained on ImageNet-1K, where test accuracy and model weights are provided. We set $\lambda^* = m / \Tr(\WW^T \WW)$ for each model. In \cref{fig:free_energy_vs_test_acc} (left), we present \cref{eq:bfe_weights} versus top-5 test accuracy, where we see a strong Spearman correlation of $\mathbf{\rho=-0.71}$ across groups\footnote{There is also strong correlation within groups.}. 

\paragraph{Large language models (LLMs).} 
We employ both the Pythia project \citep{biderman2023pythia} models ($410$M, $1$B, $1.4$B, $2.8$B, $6.9$B, and $12$B), as well as the Qwen3.5 suite \citep{qwen3.5} ($2$B, $4$B, $9$B, $35$B, $122$B, and $397$B). In \cref{fig:free_energy_vs_test_acc} (middle) we compare the dual \cref{eq:bfe_weights} against LAMBADA \citep{paperno2016lambada} test accuracy for the Pythia models, where we see a solid correlation of $\mathbf{\rho=-0.66}$. In \cref{fig:free_energy_vs_test_acc} (right), we repeat this for the Qwen3.5 models, using the MMLU-Pro benchmark \citep{wang2024mmlu}; here we see a stronger correlation of $\mathbf{\rho = -0.77}$. This is strong evidence that for both CV models and LLMs, \cref{eq:bfe_weights} is, perhaps surprisingly, strongly predictive of model quality. 

We also compare against the HT-SR metrics in \cref{sec:appendix:WW} for the CV models, where \cref{eq:bfe_weights} shows much stronger correlation. Note that for CV, two model classes did not show correlation between \cref{eq:bfe_weights} and accuracy. 
Firstly, models that used \textit{dropout} before the final-layer weight did not show correlation, as dropout prevents complex correlations from occurring in final-layer weights \citep{srivastava2014dropout}. 
Secondly, vision transformers did not show correlation.
This could be because division of vision transformers at the last-layer into \textit{features} and \textit{weights} may be too crude. 
See \cref{sec:appendix:bad_free_energy} for results showing this (alongside HT-SR metrics, which perform as poorly).
The extension to such models is left for future work.
\vspace{-2mm}

\section{Heavy-tails and free energy}
\label{sec:heavy_tail}
\vspace{-2mm}
In the previous section, \cref{eq:bfe_weights} was derived for the final-layer weights.
In this section, we will relax this condition, and consider \cref{eq:bfe_weights} for all weights of a DNN. 
We will observe in the sequel that \cref{eq:bfe_weights} for deeper DNN weights can still correlate very well with model quality, and we will give theoretical motivation (in \cref{sec:appendix:deeper_weights}). 
Thus, equipped with a metric, \cref{eq:bfe_weights}, that explicitly connects trained weights of a DNN to model quality, we will use the machinery of RMT to investigate the question of why---from a statistical perspective---the spectral shape of heavy-tailed weights is so strongly predictive of model performance. 
We give details on the spectral shape of \textbf{trained weights}; we then use this to analyze how \cref{eq:bfe_weights} is optimized for certain spectral parameters, under an RMT limit; and we conclude with empirical validation for large-scale DNNs. 

\begin{figure}[t]
    \centering
    \begin{subfigure}[b]{0.49\linewidth}
        \includegraphics[width=\linewidth]{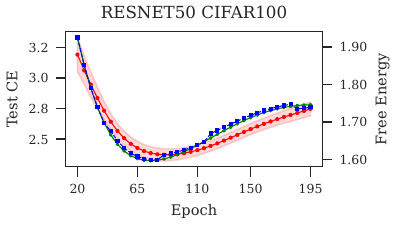}
    \end{subfigure}
    \begin{subfigure}[b]{0.49\linewidth}
        \includegraphics[width=\linewidth]{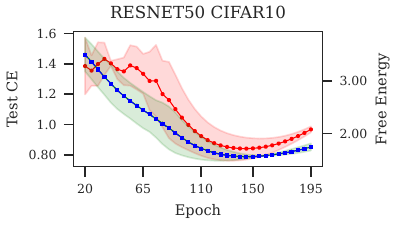}
    \end{subfigure}
    \begin{subfigure}[b]{0.49\linewidth}
        \includegraphics[width=\linewidth]{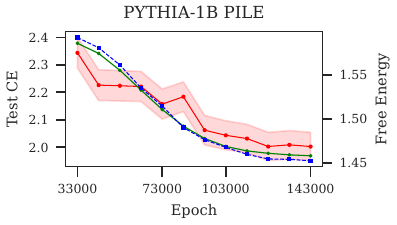}
    \end{subfigure}
    \begin{subfigure}[b]{0.49\linewidth}
        \includegraphics[width=\linewidth]{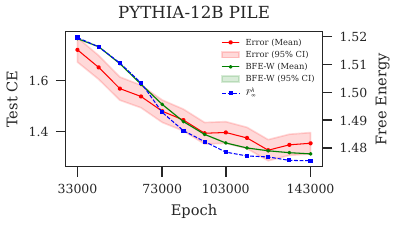}
    \end{subfigure}
    \caption{Test CE, \cref{eq:bfe_weights} and $\sF_{\infty}^{\lambda}$, vs. epochs. Model quality during training is highly correlated with BFE. Further, $\sF_{\infty}^{\lambda}$ closely approximates \cref{eq:bfe_weights}.}
    \label{fig:free_energy_training}
\end{figure}

\subsection{Trained weight spectra}
The spectrum of weight matrices before and after training differs substantially \citep{martin2021implicit}. The ESD of weight matrices $\WW_l^t \WW_l$, where $\WW_l \in \real^{d_{l-1} \times d_l}$, for DNNs at \textbf{initialization}\footnote{Under the standard initialization schemes for DNNs (e.g., Xavier, He, Glorot initializations).} converges\footnote{As $d_l, d_{l-1} \to \infty$ such that $d_l / d_{l-1} \to \gamma \in (0,\infty)$.} to the Marchenko-Pastur distribution ($\textbf{MP}_\gamma$), with aspect ratio $\gamma = d_l / d_{l-1}$. After training, the spectrum of weight matrices is often well-modeled by the HTMP distribution (see \citet[Proposition 5.3]{hodgkinson2025models}), which admits heavy-tails.

\begin{assumption}
\label{ass:main}
For $\WW_l \in \real^{d_{l-1} \times d_l}$ (e.g., Convolutional, Dense, etc.) in a DNN, in the limit as $d_l, d_{l-1} \to \infty$, $d_l/d_{l-1} \to \gamma \in (0,1)$, the spectral distribution of $\WW_l^T \WW_l -  \tau(d_{l}) \eye$ converges weakly in probability to the $\textbf{HTMP}_{\kappa, \gamma, \beta}$ distribution, where $\min \spec (\WW_l^T \WW_l) =: \tau(d_{l}) \to \tau$. 
\end{assumption}
\begin{definition}
    If $\spec \left( \WW_l^T \WW_l -  \tau \eye \right) \sim \mathbf{HTMP}_{\kappa, \gamma, \beta}$, then $\spec \left( \WW_l^T \WW_l\right) \sim \mathbf{HTMP}_{\kappa, \gamma, \beta, \tau}$.
\end{definition}
See \cref{fig:esd_training} for empirical validation of this assumption for large DNNs. We provide further details on the $\textbf{HTMP}_{\kappa, \gamma, \beta, \tau}$ distribution in \cref{sec:appendix:htmp}, as well as some mild integrability assumptions. 

\vspace{-2mm}
\subsection{Limiting free energy}

We now find the deterministic limit of \cref{eq:bfe_weights} under RMT scaling. This quantity will allow us to connect spectral shape qualities of trained weights to model quality, through the empirically observed correlation between \cref{eq:bfe_weights} and model quality in \cref{sec:free_energy_experiment}.

\begin{theorem}\label{thm:main}
Under \cref{ass:main}, \cref{ass:int_1}, and \cref{ass:int_2},
\begin{align}
    &\lim_{d_{l-1}, d_l \to \infty}  \mathcal{F}^{\lambda}_{d_{l}} \overset{\mathbb{P}}{\to} \sF_{\infty}^{\lambda} := \frac{\lambda}{2} \left(\frac{\kappa}{2 \gamma \beta} +\tau\right)+ \frac{1}{\kappa} \log U \left(\frac{\kappa}{2}, 1 - \frac{\kappa}{2\gamma} + \frac{\kappa}{2}; \beta \tau \right)  + \frac{1}{2} \log \frac{2 \pi \beta}{\lambda}, \label{eq:limit_bfe}
\end{align}
where $U(a,b;z)$ is the Tricomi confluent hypergeometric function.
\end{theorem}

The proof of \cref{thm:main} is found in \cref{sec:appendix:proof_main}. Now, this deterministic limit depends non-trivially on several parameters. 
Hence, we would like to find a parameter combination through which analysis is simplified. 
To this end, and to investigate the question of why heavier-tails correspond to improved model quality, note that the density $\bar{\rho}_{\gamma, \kappa, \beta}$ of $\textbf{HTMP}_{\kappa, \gamma, \beta}$ follows
\begin{align*}
    \bar{\rho}_{\gamma, \kappa, \beta}(x) \propto x^{\frac{\kappa}{2 \gamma} - 1 + \frac{\kappa}{2}} e^{-\beta x}, \quad x \to \infty.
\end{align*}
Hence, under $\textbf{HTMP}_{\gamma, \kappa, \beta}$, either increasing $\kappa$ or decreasing $\beta$ (while keeping the other parameter fixed) will give \textit{heavier} tails. A natural quantity to consider is therefore the scale-based heaviness $h(\kappa) = \kappa / \beta$. We now present a result which states that there exists, under certain assumptions, an optimal scale-based heaviness $h^*$; once this optimal $h$ is achieved, increasing $\kappa$ will decrease the limiting free energy. The proof is given in \cref{sec:appendix:convex_proof}.
\begin{figure}[t]
    \centering
    \begin{subfigure}[b]{0.49\textwidth}
        \includegraphics[width=\textwidth]{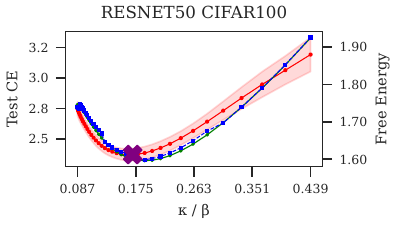}
    \end{subfigure}
    \begin{subfigure}[b]{0.49\textwidth}
        \includegraphics[width=\textwidth]{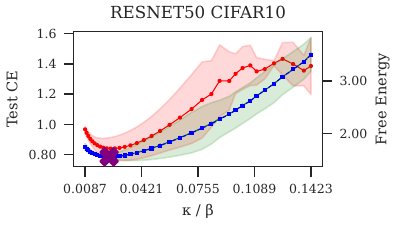}
    \end{subfigure}
    \begin{subfigure}[b]{0.49\textwidth}
        \includegraphics[width=\textwidth]{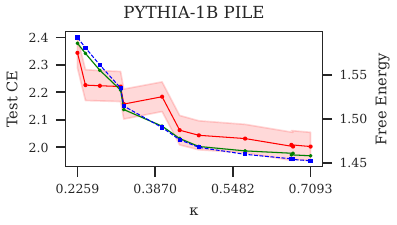}
    \end{subfigure}
     \begin{subfigure}[b]{0.49\textwidth}
        \includegraphics[width=\textwidth]{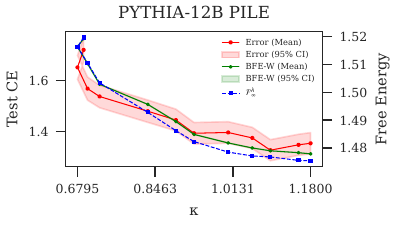}
    \end{subfigure}
    \caption{Test error, \cref{eq:bfe_weights} and $\sF_{\infty}^{\lambda}$ vs. fitted $\kappa / \beta$ and $\kappa$. For CV, $h^*$ (purple cross) aligns with minimum error. For LLMs, $h^*$ is attained early, and increasing $\kappa$ lowers test error.}
    \label{fig:free_energy_training_htmp}
\end{figure}

\begin{theorem}\label{theorem:first_moment_shannon}
For $\sF_{\infty}^{\lambda}$ in \cref{eq:limit_bfe}, we have the following results:
\begin{description}
    \item[a.] For fixed $\kappa > 0$, the first moment $\mu_1$ and the Shannon transform $\mathcal{V}_{\rho}$ (defined in \cref{def:shannon}) are both increasing in $h(\kappa)$.
    \item[b.] For fixed $\kappa > 0$, and for $\tau \to 0$, $\sF_{\infty}^{\lambda}$ is strictly convex and coercive in $h$, and there exists a unique optimal ratio $h^*(\kappa) = 2\gamma/\lambda$ for which $\sF_{\infty}^{\lambda}$ reaches its minimum.
    \item[c.] For $\tau \to 0$, at $\beta = \beta^*(\kappa) = \lambda \kappa / 2\gamma$, $\sF_{\infty}^{\lambda}$ is strictly decreasing in $\kappa$.
\end{description}
\end{theorem}
Under \cref{ass:main}, \cref{ass:int_1} and \cref{ass:int_2}, we see that the competing terms in \cref{eq:bfe_weights} are increasing and decreasing, respectively, with $h$, due to the following identities:
\begin{align*}
    \lim_{d_{l-1}, d_l \to \infty} \frac{1}{d_{l}} \Tr( \WW_l^T \WW_l) \overset{\mathbb{P}}{\to} \mu_1, \quad \lim_{d_{l-1}, d_l \to \infty} \frac{1}{d_{l}} \log \det( \WW_l^T \WW_l) \overset{\mathbb{P}}{\to} \lim_{x \to 0}  (\sV(x) + \log (x)).
\end{align*}
In practice, the spectral gap is non-zero, and hence we must impose further assumptions on the parameterization of $\kappa$, $\beta$, $\tau$. We see in \cref{fig:htmp_params} that for CV models, $\beta$ and $\tau$ are approximately inversely proportional over training epochs; under this assumption, we have the following lemma, whose proof is given in \cref{sec:appendix:opt_ratio}.

\begin{lemma}\label{lemma:opt_ratio}
    Under $\textbf{HTMP}_{\gamma, \kappa, \beta, \tau}$, if $\beta = c/\tau$ for $c > 0$, then $h^*(\kappa) = 2\gamma\kappa/\lambda(\kappa + 2\gamma c)$.
\end{lemma}

\subsection{Using HTMP to understand training dynamics}\label{sec:htmp_dynamics}
We now employ the theoretical results above to explain learning dynamics for large DNNs. Further experimental results can be found in \cref{sec:appendix:further_experiments}.
\begin{figure}
    \centering
    \includegraphics[width=\linewidth]{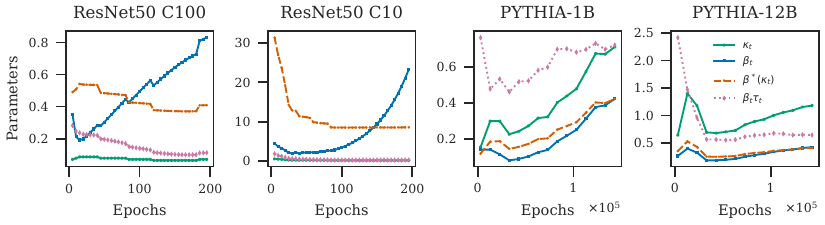}
    \caption{Fitted $\textbf{HTMP}_{\gamma, \kappa, \beta, \tau}$ parameters at each epoch, for ResNet50 (trained on CIFAR-10 and CIFAR-100), Pythia-1B and Pythia-12B models. All weights are final-layer weights, except for ResNet50 CIFAR-10, where flattened convolutional weight at the third layer is used.}
    \label{fig:htmp_params}
\end{figure}
\vspace{-2mm}
\paragraph{CV.} 
We consider ResNet50, trained on both CIFAR-100 and CIFAR-10, for $200$ epochs. After reaching zero training loss ($99\%$ accuracy is reached at epoch $20$ and $25$ respectively), each training epoch corresponds to a new \textit{trained} model. In \cref{fig:free_energy_training}(top) we plot the test cross-entropy versus \cref{eq:bfe_weights} and $\sF_{\infty}^\lambda$, for the final-layer weights (CIFAR-100) and the flattened convolutional weight at the third layer (CIFAR-10, due to small dimension of final weights). We take $\lambda = d_l/||\WW_{l,t^*}||_F^2$, where $t^*$ is the epoch when test error is minimized. Interestingly, taking $\lambda = d_l / ||\WW_t||_F^2$ at each epoch $t$ gives poor correlation for these CV models (see \cref{sec:appendix:lambda} for an in-depth discussion about the choice of $\lambda$). Using $\lambda = d_l/\Tr(\WW_{l,t^*}^T \WW_{l,t^*})$, we observe that \cref{eq:bfe_weights} and $\sF_{\infty}^\lambda$ very closely follow the test error curve, with correlations of $\mathbf{\rho = 0.927}$ and $\mathbf{\rho = 0.984}$, respectively. 

In \cref{fig:htmp_params} (left), $\kappa$ stays mostly fixed for each model, while $h_t = \kappa_t / \beta_t$ decreases during training epochs $t$. In \cref{fig:free_energy_training_htmp}, we plot test error, and the empirical and asymptotic free energy, against $h$. We plot the optimal ratio $h^*$ from \cref{lemma:opt_ratio}, and observe that the minimum test error occurs when $h_t = h^*$. Thus, $h$ is directly measuring the level of underfitting / overfitting of the model. 

\paragraph{LLM.} 
We repeat this process for the Pythia-$1$B and Pythia-$12$B models, trained on Pile. We plot the empirical and asymptotic free energies using the \textit{EmbedOut} weights of the model (final-connected weights). We plot the free energy against the cross-entropy for the OpenAI curated LAMBADA test dataset \citep{paperno2016lambada}. We compute $\lambda$ using the same process as for the CV models above. We see again that we have very close adherence of the free energies with the test error, with correlations of $\mathbf{\rho = 0.986}$ and $\mathbf{\rho = 0.965}$, respectively. Note that in all figures, \textit{epochs} of training for Pythia models refers to \textit{steps} of training.

The spectral behaviour of these LLMs contrasts that of the CV models; in \cref{fig:htmp_params} (middle, right), we observe that $\beta_t$ follows very closely the optimal $\beta^*(\kappa_t)$ from \cref{lemma:opt_ratio} (remembering that $h^*(\kappa) = \kappa / \beta^*$), while $\kappa_t$ increases during training. As predicted in \cref{theorem:first_moment_shannon} for small $\tau$, we see in \cref{fig:free_energy_training_htmp} (bottom) that this increase in $\kappa_t$ results in a decreasing free energy, which correlates with a decrease in test error. This dual optimization is distinct from the training dynamics for CV models, suggesting an inherent difference in learning dynamics.

\section{Conclusion \& discussion} 
\label{sec:conclusion}
HT-SR theory posits that the spectral distributions of DNN weight matrices can predict model quality. 
In this work, we validate this claim through a probabilistic framework; we derive a data-free metric, the Bayes Free Energy evaluated using only the final-layer weights, and we empirically show that this metric is strongly predictive of model quality, for both CV models and LLMs. 
To understand \emph{why} heavy-tailed weight matrices may generalize better, 
we derive the deterministic limiting quantity of this metric under the assumption that weight spectra follow the recently proposed
$\textbf{HTMP}_{\kappa, \gamma, \beta, \tau}$ distribution. 
We then show that this free energy is convex in the scale-based heaviness parameter, $h = \kappa/\beta$, and that an optimal parameter $h^*$ exists. Once this optimal parameter is attained, we show increasing $\kappa$ monotonically decreases free energy (correlating with improved model quality). We empirically find that CV model training correlates with optimizing $h$; in contrast, LLM models optimize $h$ early, and instead increasing $\kappa$ correlates with improved performance. Thus, we provide results that mirror the findings of the HT-SR theory; that optimal model quality can be observed in the shape of weight spectra.
\paragraph{Limitations \& future work.}
There are several limitations of our work; some of these limitations present clear paths for future work. Firstly, we note that \cref{eq:bfe_weights} was only derived for final-layer weights, and for regression. Further, the dual formulation of \cref{eq:bfe_weights} cannot immediately be compared across models with different activation dimension $d_L$. Providing a remedy to these limitations is a promising avenue for future work. Further work could also involve extending \cref{eq:bfe_weights} to models that use Dropout before the final-layer, or Vision Transformer models, as these models did not show correlation between \cref{eq:bfe_weights} and model quality. 

We also note that a limitation of our framework is that for CV models during training, the prior precision parameter $\lambda$ must be tuned post-hoc to optimize correlation. Deriving $\lambda$ a-priori would allow one to measure $\textbf{HTMP}_{\kappa, \gamma, \beta, \tau}$ parameters during training, and to early-stop when $h_t \approx h^*$. Further future work should also involve attempting to theoretically derive the $\textbf{HTMP}_{\kappa, \gamma, \beta, \tau}$ parameters for certain classes of DNNs, or to show how certain training mechanisms affect these spectral parameters. Our framework would then elucidate how these mechanisms optimize model quality, through the free energy connection.

\begin{ack}
Liam Hodgkinson is supported by the Australian Research Council through a Discovery Early Career Researcher Award (DE240100144).
Fred Roosta was partially supported by the Australian Research Council through an Industrial Transformation Training Centre for Information Resilience (IC200100022). Zhichao Wang is supported by the National Science Foundation under
Grant No. DMS-1928930, as a resident at the Simons Laufer
Mathematical Sciences Institute in Berkeley, California, in
Spring 2025. Michael W. Mahoney acknowledges partial support from
DARPA, DOE, NSF, and ONR.
We also thank Charles Martin for helpful conversations.
\end{ack}

\bibliography{Bib/biblio}
\bibliographystyle{Bib/iclr2027_conference}

\newpage
\appendix
\crefname{appendix}{Appendix}{Appendices}
\Crefname{appendix}{Appendix}{Appendices}

\section*{\Large Appendices}

\startcontents[appendix]
\printcontents[appendix]{}{1}{}

\section{Derivation of feature marginal likelihood for ridge-regression} \label[appendix]{sec:appendix:marginal}
In this section, we first derive the marginal (Gibbs) likelihood for features $\Phi$, under the training framework detailed in \cref{sec:ridge_regression}. We then provide the Bayes Free Energy, which is the negative log of the marginal likelihood. 

\subsection{Marginal likelihood derivation}

The marginal (Gibbs) likelihood in this setting is given by $\sZ^{\tau}_n(\Phi) = \Ex_{\pi_{\WW}} \exp(-L(\WW, \Phi) / \tau)$, for loss function $L(\WW, \Phi) = ||\Phi \WW - \YY||_F^2$, and $\pi_{\WW} = \sM \sN ( \bm 0, \sigma^2\eye_{d_L},\eye_m)$. As $q(\Phi)$ is only defined up to constants of proportionality, we redefine the marginal likelihood as 
\begin{align*}
    \sZ^{\tau}_n(\Phi) = (\pi \tau)^{-nm/2} \Ex_{\pi_{\WW}} \exp(-L(\WW, \Phi) / \tau).
\end{align*}
This will simplify the calculations. 

\begin{remark}
    The marginal (Gibbs) likelihood for this ridge-regression setting was derived in \citet{hodgkinson2025models}; however, we provide a different derivation, using integrals of matrix-normals. Further, \citet{hodgkinson2025models} provides the derivation of $\sZ^{\tau}_n(\Phi)$ for $L(\WW, \Phi) = ||\Phi \WW - \YY||_F^2 + \mu||\WW||_F^2$. We take $\mu = 0$.
\end{remark}

Now, note that $(\pi \tau)^{-nm/2} \exp(-||\Phi \WW - \YY||_F^2 / \tau)$ is the likelihood of the model $\YY | \WW, \Phi \sim \sN(\Phi \WW, \tau / 2 \eye)$. Then, 
\begin{align*}
    \sZ^{\tau}_n(\Phi) = \Ex_{\pi_{\WW}} p(\YY | \WW, \Phi) = \int_{\supp(\pi_{\WW})} p(\YY | \WW, \Phi) p(\WW) d\WW = p(\YY | \Phi).
\end{align*}

We can write the model as $\YY = \Phi \WW + \bm\epsilon$, where $\WW \sim \sM \sN ( \bm 0, \sigma^2\eye_d,\eye_m)$ and $\epsilon \sim \sM \sN ( \bm 0, \tau/2 \eye_n,\eye_m)$. The sum of matrix normals is a matrix normal, and we are integrating out $\WW$. Thus, using standard matrix-normal properties \citep{gupta1999matrix},
\begin{align*}
    \YY = \Phi \WW + \epsilon \sim \sM \sN( \bm 0, \sigma^2 \Phi \Phi^T + \tau / 2 \eye_n, \eye_m).
\end{align*}
Therefore, 
\begin{align*}
    \sZ^{\tau}_n(\Phi) = p(\YY | \Phi) &= (2\pi)^{-nm/2} \frac{\exp\left(-\frac{1}{2} \Tr \left( \YY^T \left(\sigma^2 \Phi \Phi^T + \tau / 2 \eye_n \right)^{-1} \YY \right)\right)}{\det \left( \sigma^2 \Phi \Phi^T + \tau / 2 \eye_n\right)^{m/2}} \\
    &= (2\pi\sigma^2)^{-nm/2} \frac{\exp\left(-\frac{1}{2\sigma^2 } \Tr \left( \YY^T \left(\Phi \Phi^T + \frac{\tau}{2\sigma^2} \eye_n \right)^{-1} \YY \right)\right)}{\det \left(\Phi \Phi^T + \frac{\tau}{2\sigma^2} \eye_n\right)^{m/2}}.
\end{align*}

\subsection{Bayes free energy for regression}
We define the Bayes Free Energy as the negative log of the marginal likelihood, i.e., $\sF^{\tau}_n(\Phi) := - \log \sZ^{\tau}_n(\Phi)$. Taking $\lambda_m = 1 / m\sigma^2$, and $\tau_m = \tau m$, we can compute that
\begin{align*}
    \frac{1}{nm} \sF^{\tau_m, \lambda_m}_n(\Phi) &:= \frac{\lambda_m}{2n} \Tr \left( \YY^T \left(\Phi \Phi^T + \frac{\lambda_m \tau_m}{2} \eye_n \right)^{-1} \YY \right) + \frac{1}{2n}\log \det \left(\Phi \Phi^T + \frac{\lambda_m \tau_m}{2} \eye_n\right) \\
    &+ \frac{1}{2}\log \left(\frac{2\pi}{\lambda_m m} \right).
\end{align*}

\section{Derivation of BFE-W}
\label[appendix]{sec:appendix:weight_energy}

In this section, we first prove \cref{thm:bfe_w}, and then derive the dual formulation. To prove these results, we rely on asymptotic expansions of $\Tr ((A + \nu \eye)^{-1})$ and $\log \det (A + \nu \eye)$, for singular matrix $A$ and perturbative parameter $\nu \ll 1$.

\subsection{Primal form}
\begin{proof}[Proof of \cref{thm:bfe_w}]\label[appendix]{proof:bfe}
Under our rank assumptions and interpolation assumptions, $\WW = \Phi^{\dagger} \YY$ is equivalent to $\Phi = \YY \WW^{\dagger}$. Let $\nu = \lambda_m \tau_m / 2 = \tau / 2\sigma^2$. We consider the trace and log-determinant terms in \cref{eq:bfe_phi} separately.

For the \textbf{trace term}, for $\nu \ll 1$, we get the following singular asymptotic expansion around $\nu$ for positive semi-definite $\AA \in \real^{n \times n}$ \citep{gonzalez2015laurent},
\begin{align*}
    (\AA + \nu \eye)^{-1} = \frac{1}{\nu} (\eye - \AA \AA^{\dagger}) + \AA^{\dagger} + \mathcal{O(\nu)}.
\end{align*}
Noting that $(\eye - \AA \AA^{\dagger})$ is a projection onto the kernel of $\AA$, and noting that $\YY \in \range(\Phi) = \range(\Phi \Phi^T)$, we get that
\begin{align*}
    \Tr \left( \YY^T \left( \Phi \Phi^T + \nu \eye\right)^{-1} \YY \right) = \Tr \left( \YY^T \left( \Phi \Phi^T \right)^{\dagger} \YY \right) + \mathcal{O(\nu)} = \Tr \left( \WW^T \WW \right) + \mathcal{O(\nu)}.
\end{align*}

For the \textbf{log-determinant term}, using the Weinstein-Aronszajn identity: for $\AA \in \real^{n \times m}, \BB \in \real^{m \times n}$, $\det(\AA\BB + \nu \eye_n) = \nu^{n-m} \det(\BB \AA + \nu \eye_m)$. and the Taylor expansion of the log-determinant of an invertible matrix around $\eye$, we see that
\begin{align*}
    &\log \det(\Phi \Phi^T + \nu \eye_n) \\
    &= \log \det (\YY (\WW^T \WW)^{\dagger} \YY^T + \nu \eye_n) \\
     &= (n-m) \log \nu + \log \det ( \YY^T \YY (\WW^T \WW)^{-1} + \nu \eye_m) \\
     &= (n-m) \log \nu + \log \det ( \YY^T \YY (\WW^T \WW)^{-1}) + \log \det (\eye_m + \nu(\YY^T \YY (\WW^T \WW)^{-1})^{-1}) \\
     &= (n-m) \log \nu + \log \det ( \YY^T \YY (\WW^T \WW)^{-1}) +  \nu \Tr((\YY^T \YY (\WW^T \WW)^{-1})^{-1}) + \mathcal{O}(\nu^2).
\end{align*}
Combining these terms, we get the following asymptotic expansion for \cref{eq:bfe_phi}, for $\tau \to 0^{+}$, evaluated using the weights
\begin{align*}
    \frac{1}{nm}\sF_n^{\tau_m, \lambda_m}(\Phi) &= \frac{\lambda_m}{2 n} \Tr (\WW^T \WW) - \frac{1}{2n} \log \det \left(\WW^T \WW\right) + \frac{1}{2} \log \frac{2 \pi}{\lambda_m m} \\
    &+ \frac{1}{2n} \log \det \Sigma + \frac{n-m}{2n}\log \left(\frac{\tau \lambda_m m}{2}\right) + \mathcal{O}(\tau).
\end{align*}
\end{proof}

\subsection{Dual form}\label[appendix]{sec:appendix:dual_derive}
If $m > d_L$, then \cref{ass:features} is invalid. Instead, we take the following, very similar, set of assumptions.
\begin{assumption}[Dual Interpolation]\label{ass:features_dual}
    For \textbf{trained} features $\Phi$, assume $n \geq m, m > d_L \geq c$ and $\rank(\Phi) = \rank(\YY) = c$. Further, assume that \textbf{trained} weights are given by $\WW = \Phi^{\dagger} \YY$, and that $\Phi \WW = \YY$.
\end{assumption}

Note that we allow for rank $c \leq d_L$, to allow for comparison across different models, with different $d_L$. This point will be expanded upon later. We then have the following lemma, whose statement and proof follows very closely that of \cref{thm:bfe_w}. 

\begin{lemma}\label{lemma:bfe_w_dual}
    Define $\Sigma := \YY^T \YY$. Then, under \cref{ass:features_dual}, as $\tau \to 0^+$,
    \begin{align*}
        \frac{1}{nm}\sF_n^{\tau_m, \lambda_m}(\Phi) &= \frac{\lambda_m}{2 n} \Tr (\WW \WW^T) - \frac{1}{2n} \log \pdet (\WW \WW^T) + \frac{1}{2} \log \frac{2 \pi}{\lambda_m m} \\
        &+ \frac{1}{2n} \log \pdet \Sigma + \frac{n-c}{2n}\log\left(\frac{\tau \lambda_m m}{2}\right) + \mathcal{O}(\tau).
\end{align*}
\end{lemma}

Before proving \cref{lemma:bfe_w_dual}, we require a lemma on factorization of pseudo-determinants.

\begin{lemma}[Pseudo-Determinant Factorization]\label{lemma:pdet}
    Let $\AA, \BB \in \real^{d \times d}$ be symmetric positive semi-definite matrices with $\rank(\AA) = \rank(\BB) = c$ and $\range(\AA) = \range(\BB)$. Then, 
    \begin{align*}
        \pdet(\AA \BB) = \pdet(\AA) \pdet(\BB).
    \end{align*}
\end{lemma}
\begin{proof}
    Take the compact spectral decompositions
    \begin{align*}
        \AA = \UU \Lambda_{\AA} \UU^T,
        \qquad
        \BB = \VV \Lambda_{\BB} \VV^T,
    \end{align*}
    where $\UU,\VV\in\real^{d\times c}$ have orthonormal columns, and $\Lambda_{\AA},\Lambda_{\BB}\in\real^{c\times c}$ are diagonal matrices containing the non-zero eigenvalues of $\AA$ and $\BB$, respectively. Since $\range(\AA)=\range(\BB)$, the columns of $\UU$ and $\VV$ are orthonormal bases for the same subspace. Hence, there exists an orthogonal $\RR\in\real^{c\times c}$ such that $\UU=\VV\RR$. Therefore,
    \begin{align*}
        \AA\BB
        =
        \VV\RR\Lambda_{\AA}\RR^T\Lambda_{\BB}\VV^T.
    \end{align*}
    Since $\RR\Lambda_{\AA}\RR^T\Lambda_{\BB}$ is non-singular, its eigenvalues
    are precisely the non-zero eigenvalues of $\AA\BB$. Thus,
    \begin{align*}
        \pdet(\AA\BB) = \det(\RR\Lambda_{\AA}\RR^T\Lambda_{\BB}) = \det(\RR)^2 \det(\Lambda_{\AA})\det(\Lambda_{\BB}) = \pdet(\AA)\pdet(\BB),
    \end{align*}
    where we used $\det(\RR)^2=1$.
\end{proof}

We now prove the dual expansion of \cref{eq:bfe_phi}.

\begin{proof}[Proof of \cref{lemma:bfe_w_dual}]
Under the rank and interpolation assumptions, $\WW = \Phi^{\dagger} \YY$ is equivalent to $\Phi = \YY \WW^{\dagger}$, and $\text{Range}(\YY^T \YY) = \text{Range}(\WW^T \WW)$. Let $\nu = \lambda_m \tau_m / 2 = \tau / 2\sigma^2$. We consider the trace and log-determinant terms in \cref{eq:bfe_phi} separately.

For the \textbf{trace term}, the exact derivation from the proof of \cref{thm:bfe_w} can be used, to get that
\begin{align*}
    \Tr \left( \YY^T \left( \Phi \Phi^T + \nu \eye\right)^{-1} \YY \right) = \Tr \left( \YY^T \left( \Phi \Phi^T \right)^{\dagger} \YY \right) + \mathcal{O(\nu)} = \Tr \left( \WW \WW^T \right) + \mathcal{O(\nu)}.
\end{align*}

For the \textbf{log-determinant term}, we note the following spectral identity for matrix $\AA \in \real^{n \times n}$, where $\rank(\AA) = c$:
\begin{align*}
    \det(\AA + \nu \eye_n) = \nu^{n-c} \prod_{i=1}^c (\lambda_i + \nu) = \nu^{n-c} \prod_{i=1}^c (\lambda_i) \prod_{i=1}^c \left(1 + \frac{\nu}{\lambda_i}\right) = \nu^{n-c} \pdet(\AA)\prod_{i=1}^c \left(1 + \frac{\nu}{\lambda_i}\right).
\end{align*}
Thus, also using the expansion $\log(1 + x) = x + \mathcal{O}(x^2)$ for small $x$,
\begin{align*}
    \log \det(\Phi \Phi^T + \nu \eye_n) &= \log \det (\YY (\WW^T \WW)^{\dagger} \YY^T + \nu \eye_n) \\
    &= (n-c) \log \nu + \log \pdet ( \YY^T \YY (\WW^T \WW)^{\dagger}) + \mathcal{O}(\nu) \\
    &= (n-c) \log \nu + \log \pdet ( \YY^T \YY) - \log \pdet (\WW \WW^T) + \mathcal{O}(\nu) \quad (\text{by \cref{lemma:pdet}}).
\end{align*}

Combining these terms, we get the following asymptotic expansion for \cref{eq:bfe_phi}, for $\tau \to 0^{+}$, evaluated using the weights
\begin{align*}
    \frac{1}{nm}\sF_n^{\tau_m, \lambda_m}(\Phi) &= \frac{\lambda_m}{2 n} \Tr (\WW \WW^T) - \frac{1}{2n} \log \pdet (\WW \WW^T) + \frac{1}{2} \log \frac{2 \pi}{\lambda_m m} \\
    &+ \frac{1}{2n} \log \pdet (\Sigma) + \frac{n-c}{2n}\log (\frac{\tau \lambda_m m}{2}) + \mathcal{O}(\tau).
\end{align*}
\end{proof}

Note that unlike the primal formulation, we cannot easily compare the dual expansion across models with differing $d_L$. For models trained on the same dataset, with the same final activation dimension $d_L$, and full-rank features (i.e., $c = d_L$), taking the difference of the asymptotic expansions in \cref{lemma:bfe_w_dual}, and then re-normalizing by $d_L$, will result in the following \textbf{relative dual free energy:} 
\begin{tcolorbox}
\vspace{-3mm}
    \begin{align}\label{eq:bfe_weights_dual}
    \tilde{\mathcal{F}}_{d_L}^{\lambda}(\WW) := \frac{\lambda}{2 d_L} \Tr (\WW \WW^T) - \frac{1}{2d_L} \log \det \WW \WW^T + \frac{1}{2} \log \frac{2 \pi}{\lambda}. \tag{DUAL BFE-$\WW$}
\end{align}
\end{tcolorbox}
However, in \cref{fig:free_energy_vs_test_acc}, we see that even for models with \textbf{different} activation dimension $d_L$, this dual free energy shows strong correlation with model quality, for models in the Pythia suite of LLMs and the Qwen3.5 suite. Deriving an asymptotic expansion that allows for direct comparison across models with differing activation dimension is left for future work.

\section{High-Temperature Marchenko-Pastur distribution}\label[appendix]{sec:appendix:htmp}
In this section, we give details on the HTMP distribution \citep{hodgkinson2025models}. 
Take a random matrix $A_N$ with beta-Laguerre ensemble of size $N \times N$, defined by its joint eigenvalue density function
\begin{align}
    q(\lambda_1, \dots, \lambda_N) \propto \prod_{i = 1}^{N} \lambda_i^{\frac{\tilde{\beta}}{2}(D - N + 1) - 1}e^{-\frac{\tilde{\beta} D}{2} \lambda_i} \prod_{i,j=1,\dots,N,~ i<j} |\lambda_j - \lambda_i|^{\tilde{\beta}}, \label{eq:betaLaguerre}
\end{align}
for some $\tilde{\beta} = \tilde{\beta}(N) > 0$, and $D, N \in \mathbb{N}, $where $D = D(N) \geq N$. Let $\kappa = \kappa(N) = \tilde{\beta}(N) N$. As the ensemble size $N \to \infty$, the empirical spectral distribution (ESD) of $A_N$ will converge to either the Marchenko-Pastur (MP) distribution, or the HTMP distribution, depending on the limit of $\tilde{\beta}(N)$. We first consider the case when $\tilde{\beta}(N) N \to \infty$ (equivalently $\kappa(N) \to \infty$), where the ESD of $A_N$ converges to the MP distribution, with probability density function 

\begin{align}
    \rho_\gamma(x) = (1 - \gamma^{-1})^+ \delta_0(x) + \frac{1}{2\pi} \frac{\sqrt{(\gamma_+ - x)^+(x - \gamma_-)^+}}{\gamma x}, \quad x \geq 0 \label{eq:MP} \tag{\textbf{MP}},
\end{align}
where $\gamma_{\pm} = (1 \pm \sqrt\gamma)^{2}$, and $(x)^+ = \max (0,x)$. In this case, we have the following result.
\begin{theorem}[\citet{hodgkinson2025models}]
If $\tilde{\beta}(N) N \to \infty$ as $N \to \infty$ and $N/D \to \gamma \in (0,1)$, then the ESD $\mu_N$ of beta-Laguerre ensemble $A_N$, defined by \cref{eq:betaLaguerre}, converges weakly to the Marchenko-Pastur distribution $\textbf{MP}_{\gamma}$, whose probability density function is $\rho_\gamma(x)$ with parameter $\gamma$, defined in \cref{eq:MP}.
\end{theorem}

If, instead, the temperature $\tilde{\beta}(N) \sim \kappa / N$, then we arise at the high-temperature MP distribution, as stated in the following result.

\begin{theorem}[\citet{hodgkinson2025models}]
Assume $\kappa(N) = \tilde{\beta}(N)N \to \kappa \in (0,\infty)$ as $N \to \infty$ and $N/D \to \gamma \in (0,1)$, then the ESD $\mu_N$ of beta-Laguerre ensemble $A_N$, defined by \cref{eq:betaLaguerre}, converges weakly to a probability distribution on $(0,\infty)$ with probability density function $\rho_{\gamma, \kappa}$ defined as
\begin{align*}
    \rho_{\gamma, \kappa}(x) = \frac{\kappa}{2\gamma} \frac{1}{\Gamma(\kappa/2 + 1)\Gamma(\kappa/2\gamma)} \frac{(\frac{\kappa x}{2\gamma})^{\kappa/2\gamma - 1 - \kappa/2} e^{-\frac{\kappa x}{2\gamma}}}{|U(\kappa/2, -\frac{\kappa}{2\gamma} + 1 + \frac{\kappa}{2}; -\kappa x / 2 \gamma)|^2}, && x \geq 0,
\end{align*}
where $U(a,b,z)$ is the Tricomi confluent hypergeometric function \citep{DLMF}. This limiting spectrum is the $\textbf{HTMP}_{\gamma, \kappa}$ distribution, with parameter $\kappa, \gamma > 0$.
\end{theorem}

\begin{remark}
By \citet[Theorem E.1]{hodgkinson2025models}, $\textbf{HTMP}_{\kappa, \gamma, \beta} \to \textbf{MP}_{\gamma}$ as $\kappa, \beta \to \infty$ such that $\kappa / \beta \to 2\gamma$.
This---coupled with the fact that the model can exhibit heavy-tailed ESDs, while not exhibiting heavy-tailed elements---means that the $\textbf{HTMP}_{\kappa, \gamma, \beta, \tau}$ distribution should well-model most stages of DNN training.\footnote{It will not cover bulk $+$ spike distributions common for older DNNs; see \citet{martin2021implicit} for the 5+1 phases.} 
\end{remark}

The Stieltjes transform of a probability distribution $\rho$ is given by $S_\rho(z) = \int (x - z)^{-1} \rho(x) dx$. 
For the $\textbf{HTMP}_{\gamma, \kappa}$ distribution $\rho_{\gamma, \kappa}$, the \textbf{Stieltjes transform} is given by a result from \citet{ismail1988linear}, as follows:
\begin{align*}
    S_{\rho_{\gamma, \kappa}}(z) = \frac{\kappa}{2 \gamma} \frac{U(\frac{\kappa}{2} + 1, 2 - \frac{\kappa}{2 \gamma} + \frac{\kappa}{2}; -\frac{\kappa z}{2 \gamma})}{U(\frac{\kappa}{2}, 1 - \frac{\kappa}{2 \gamma} + \frac{\kappa}{2}; -\frac{\kappa z}{2 \gamma})}, && z \in \complex \setminus \real.
\end{align*}

\citet[Proposition 5.3]{hodgkinson2025models} states that under certain assumptions, as $m, d \to \infty$ such that $\gamma(m) := \frac{\kappa(m)/2}{\alpha(m) - \kappa(m)/2 - 1} \to \gamma$, the limiting ESD of $\frac{\beta}{\alpha - \kappa/2 -1} \WW^T \WW$ converges to $\textbf{HTMP}_{\gamma, \kappa}$, for parameters $\beta, \alpha > 0$ ($\beta$ here is separate from the $\tilde{\beta}$ parameter of the beta-Laguerre ensemble). We denote the limiting ESD of $\WW^T \WW$ as $\Bar{\rho}_{\gamma, \kappa, \beta}$,
where
\begin{align}
    \bar{\rho}_{\gamma, \kappa, \beta}(x) = \beta \frac{1}{\Gamma(\kappa/2 + 1)\Gamma(\kappa/2\gamma)} \frac{(\beta x)^{\kappa/2\gamma - 1 - \kappa/2} e^{-\beta x}}{|U(\kappa/2, -\frac{\kappa}{2\gamma} + 1 + \frac{\kappa}{2}; -\beta x)|^2}, && x \geq 0. \label{eq:htmp} \tag{\textbf{HTMP}}
\end{align} 

It is immediate that the Stieltjes transform for \cref{eq:htmp} is given by
\begin{align}
    S_{\bar{\rho}_{\gamma, \kappa, \beta}}(z) = \beta \frac{U(\frac{\kappa}{2} + 1, 2 - \frac{\kappa}{2 \gamma} + \frac{\kappa}{2}; -\beta z)}{U(\frac{\kappa}{2}, 1 - \frac{\kappa}{2 \gamma} + \frac{\kappa}{2}; -\beta z)}. \label{eq:StieltjesWeights}
\end{align}

\subsection{Derivation of CDF for HTMP}
We now derive the cumulative distribution function (CDF) for \cref{eq:htmp}. 

\begin{lemma}[CDF HTMP]\label{lemma:htmp_cdf}
    The CDF for the density function $\bar{\rho}_{\gamma, \kappa, \beta}(x)$ is given by
    \begin{align*}
        F_{\gamma, \kappa, \beta}(x) = 
        \begin{cases}
            0, & x < 0, \\
            \frac{2}{\pi \kappa} \lim_{y \downarrow 0} \Arg_{\text{cont}}U\left( \frac{\kappa}{2}, 1 - \frac{\kappa}{2 \gamma} + \frac{\kappa}{2}; -\beta x - i \beta y\right), & x \geq 0,
        \end{cases}
    \end{align*}
    where $\Arg_{\text{cont}}$ denotes the continuous, unwrapped argument, and $\Arg_{\text{cont}}(0) = 0$. 
\end{lemma}
\begin{proof}
    Let $a = \kappa / 2$, and $b = 1 - \kappa/2 \gamma + \kappa/2$. Note that $S_{\bar{\rho}_{\gamma, \kappa, \beta}}(z) = a^{-1} \partial_z \log U(a,b,-\beta z)$. Now we can write the CDF through the Stieltjes inversion formula \citep[Eq. 3.2]{bai1999methodologies},
    \begin{align*}
        F_{\gamma, \kappa, \beta}(x) - F_{\gamma, \kappa, \beta}(0) = \frac{1}{\pi}\lim_{y \downarrow 0} \int_0^x \Im S_{\gamma, \kappa, \beta}(t + iy) dt = \frac{1}{\pi a}\lim_{y \downarrow 0} \int_0^x \Im [\partial_t \log U(a,b,-\beta (t + iy))] dt.
    \end{align*}
    As $\bar{\rho}_{\gamma, \kappa, \beta}(x)$ is supported on $x \geq 0$, $F_{\gamma, \kappa, \beta}(0) = 0$. We write $U(a,b;z)$ in its polar form, $U(a,b;z) = R(x) \exp(i\Theta(x))$, where we take $\Theta(x)$ as the continuous argument. We can then write 
    \begin{align*}
        \Im [\partial_x \log U(a,b,-\beta (x + iy))] = \partial_x \Im [\log U(a,b,-\beta (x + iy))] = \partial_x \Theta(-\beta (x + iy)).
    \end{align*}
    Hence,
    \begin{align*}
        F_{\gamma, \kappa, \beta}(x) = \frac{1}{\pi a}\lim_{y \downarrow 0} \Theta(-\beta x - i \beta y) - \frac{1}{\pi a}\lim_{y \downarrow 0} \Theta(0- i \beta y).
    \end{align*}
    Now, $\lim_{y \downarrow 0}U(a,b;-iy) = \Gamma(1 - b) / \Gamma(a - b + 1)$ \citep{DLMF}. For $a = \kappa / 2$, and $b = 1 - \kappa/2 \gamma + \kappa/2$, $1-b > 0$ and $a - b + 1 > 0$, hence, $\lim_{y \downarrow 0}U(a,b;-iy) > 0$, so $\lim_{y \downarrow 0} \Theta(0- i \beta y) = 0$. Thus,
    \begin{align*}
        F_{\gamma, \kappa, \beta}(x) = \frac{1}{\pi a}\lim_{y \downarrow 0} \Theta(-\beta x - i \beta y).
    \end{align*}
\end{proof}

Now, in \cref{appendix:experiment:fit_htmp} we discuss fitting the HTMP by minimizing Kolmogorov-Smirnov distance, i.e. by minimizing the pointwise supremum of the HTMP CDF to the empirical CDF. To do this efficiently, it will be useful to be able to compute the \textbf{principal} argument of $\lim_{y \downarrow 0} U\left( \kappa / 2, 1 - \kappa/2 \gamma + \kappa/2; -\beta x - i \beta y\right)$; then, an unwrapping process on a grid of inputs $x$ will return the continuous argument. Hence, we present the following result, which provides the real and imaginary parts of $\lim_{y \downarrow 0} U\left( \kappa / 2, 1 - \kappa/2 \gamma + \kappa/2; -\beta x - i \beta y\right)$; further, these parts only involve functions that can be easily computed numerically (i.e. they do not require complex arguments). However, note that this form only holds when $1 - \kappa/2 \gamma + \kappa/2$ is not an integer.

\begin{lemma}[Non-Integer Evaluation of HTMP CDF]\label{lemma:cdf_noninteger}
Let $M(a, b, x) = {}_1F_1(a,b,x)$. For $a > 0$, $b < 1$, $b \notin \integer$, and $x > 0$, 
\begin{align*}
    \mathcal{R}(a,b,x) &:= \frac{\Gamma(1-b)}{\Gamma(a - b + 1)} M(a,b,-x) - \frac{\cos(\pi b) \Gamma(b-1)}{\Gamma(a)} x^{1-b} M(a - b + 1, 2 - b , -x) \\
    \mathcal{I}(a,b,x) &:=  - \frac{\sin(\pi b) \Gamma(b-1)}{\Gamma(a)} x^{1-b} M(a - b + 1, 2 - b , -x).
\end{align*}
Then, 
\begin{align*}
    \lim_{y \downarrow 0} U\left(a,b; -x - i y\right) = \mathcal{R}(a,b,x) + i \mathcal{I}(a,b,x).
\end{align*}
\end{lemma}
\begin{proof}
    By Kummer's transformation \citep[Eq. 13.2.42]{DLMF}, we can write
    \begin{align*}
        U(a,b,z) = \frac{\Gamma(1-b)}{\Gamma(a - b + 1)} M(a,b,z) + \frac{\Gamma(b-1)}{\Gamma(a)} z^{1-b}M(a - b + 1,2 - b,z).
    \end{align*}
    Kummer's function $M(a,b,z)$ is entire \citep[\S 13.2]{DLMF}; further, letting $\text{Log}(z) = \log |z| + i \arg(z)$ denote the principal logarithm,
    \begin{align*}
        \lim_{y \downarrow 0} (-x - iy)^{1-b} &= \lim_{y \to 0} \exp((1-b) \text{Log}(-x-iy)) = \exp((1-b) (\log x - i \pi)) \\
        &= x^{1-b} e^{-i\pi(1-b)} = x^{1-b} (-\cos(\pi b) - i\sin(\pi b)).
    \end{align*}
    Therefore, 
    \begin{align*}
        \lim_{y \downarrow 0} U\left(a,b; -x - i y\right) &= \frac{\Gamma(1-b)}{\Gamma(a - b + 1)} M(a,b,-x) \\
        &- \frac{\cos(\pi b) \Gamma(b-1)}{\Gamma(a)} x^{1-b} M(a - b + 1, 2 - b , -x) \\
        &- i \frac{\sin(\pi b) \Gamma(b-1)}{\Gamma(a)} x^{1-b} M(a - b + 1, 2 - b , -x).
    \end{align*}
\end{proof}

\section{Derivation of limiting free energy}
\label[appendix]{sec:appendix:main_proofs}

In this section, we will derive the limiting expression of \cref{eq:bfe_weights}, as given in \cref{thm:main}. We will assume weights $\WW_l \in \real^{d_{l-1} \times d_{l}}$, that follow \cref{ass:main}. Further, as this assumption only gives weak convergence, we require the following integrability assumptions.

\begin{assumption}[Uniform Integrability]\label{ass:int_1}
    Assume $\sup_{d_{l}} d_{l}^{-1} \Tr\left(\left(\WW_l^T \WW_l\right)^{1+\epsilon}\right) < \infty$ for some $\epsilon > 0$.
\end{assumption}
\begin{assumption}[Logarithmic Integrability]\label{ass:int_2}
    Denote $\bar{\rho}_{d_{l}}$ as the ESD for $\WW_l^T \WW_l$. Then, assume
    \begin{align*}
        \lim_{\delta \to 0^+} ~\underset{d_{l} \to \infty}{\lim \sup}\int_{0}^{\delta} |\log x| d\bar{\rho}_{d_{l}}(x) =0.
    \end{align*}
\end{assumption}
Note that \cref{ass:int_1} is a standard moment assumption in RMT. For \cref{ass:int_2}, this controls the convergence of $\log \det(\WW^T \WW)$ when $\tau \to 0$. Both assumptions are valid for non-degenerate Wishart matrices. 

To derive the limiting expression for $\sF_{d_l}^{\lambda}$ as $m \to \infty$, we require the limiting expressions for $m^{-1} \Tr \WW^T \WW$ and $m^{-1} \log \det \WW^T \WW$ as $m \to \infty$. These terms relate to the first moment and the Shannon transform of \cref{eq:htmp}.

\subsection{Shannon transform}
The limiting log-determinant of $\WW^T \WW$, under our spectral assumptions, is given by the Shannon transform of \cref{eq:htmp}. We now derive this.

\begin{definition}\label{def:shannon}
    The Shannon transform of distribution $\rho$ is given by $\mathcal{V}_{\rho}(x) := \Ex_{\lambda \sim \rho} \left [ \log \left( \frac{\lambda}{x} + 1 \right) \right ]$, for $x > 0$ \citep{tulino2004random}.
\end{definition}

\begin{lemma}[Derivation of Shannon Transform for HTMP]\label{lemma:derive_shannon}
    For $\bar{\rho}_{\gamma, \kappa, \beta}$, the Shannon transform $\mathcal{V}_{\bar{\rho}}(x)$ is given by
    \begin{align*}
        \mathcal{V}_{\bar{\rho}}(x) &= - \frac{2}{\kappa}\log \bigg [ U\bigg(\kappa/2, 1 - \frac{\kappa}{2\gamma} + \frac{\kappa}{2}; \beta x\bigg) \bigg ] -\log (\beta x), \quad x > 0.
    \end{align*}
    Further, we have the following limit,
    \begin{align*}
        \lim_{x \to 0} \left(\mathcal{V}_{\bar{\rho}}(x) + \log x \right) = -\frac{2}{\kappa} \log \left[ \frac{\Gamma(\frac{\kappa(1 - \gamma)}{2\gamma})}{\Gamma(\frac{\kappa}{2\gamma})} \right] - \log \beta.
    \end{align*}
\end{lemma}
\begin{proof}
We have the following identity \citep{tulino2004random}:
\begin{align}\label{eq:shannon}
    \mathcal{V}_{\bar{\rho}}(x) = \int_0^{\infty} \log \left(1 + \frac{\lambda}{x} \right) \Bar{\rho}_{\gamma, \kappa, \beta}(d\lambda) = \int_x^{+\infty} \left (\frac{1}{z} - S_{\Bar{\rho}_{\gamma, \kappa, \beta}}(-z) \right) dz. 
\end{align}
To compute the Shannon transform, we require the anti-derivative of the integrand in the final expression in \cref{eq:shannon}, i.e., some function $\mathcal{D}(z)$ such that $\mathcal{D}'(z) = z^{-1} - S_{\Bar{\rho}_{\gamma, \kappa, \beta}}(-z)$. Note that $\partial_z U(a,b,z) = -aU(a+1, b+1, z)$ \citep{DLMF}. From this, we can write the Stieltjes transform as follows:
\begin{align*}
   \frac{U(\frac{\kappa}{2} + 1, 2 - \frac{\kappa}{2\gamma} + \frac{\kappa}{2};\Tilde{z})}{U(\frac{\kappa}{2}, 1 - \frac{\kappa}{2\gamma} + \frac{\kappa}{2};\Tilde{z})}  &= \frac{-(\frac{\kappa}{2})^{-1}\frac{d}{d\Tilde{z}} U(\frac{\kappa}{2}, 1 - \frac{\kappa}{2\gamma} + \frac{\kappa}{2};\Tilde{z})}{U(\frac{\kappa}{2}, 1 - \frac{\kappa}{2\gamma} + \frac{\kappa}{2};\Tilde{z})} \\
    &= -(\frac{\kappa}{2})^{-1} \frac{d}{d\Tilde{z}} \log \bigg ( U \bigg(\frac{\kappa}{2}, 1 - \frac{\kappa}{2\gamma} + \frac{\kappa}{2};\Tilde{z} \bigg) \bigg ),
\end{align*}
where $\Tilde{z} = \beta z$. Hence,
\begin{align*}
    S_{\Bar{\rho}_{\gamma, \kappa, \beta}}(-z) = - (\frac{\kappa}{2})^{-1} \frac{1}{\beta} \beta \frac{d}{dz} \log \bigg ( U(\kappa/2, 1 - \frac{\kappa}{2\gamma} + \frac{\kappa}{2};\beta z) \bigg ),
\end{align*}
and thus
\begin{align}
    \mathcal{D}(z) = \log(z) + \frac{2}{\kappa} \log \bigg [ U\bigg(\kappa/2, 1 - \frac{\kappa}{2\gamma} + \frac{\kappa}{2}; \beta z\bigg) \bigg ].
\end{align}
We now need to consider the asymptotics of $\mathcal{D}$, to compute the integral in \cref{eq:shannon}. Note from \citet{DLMF}, $|U(a,b,z)| \sim |z|^{-a}$ as $|z| \to \infty$ with $|\arg z | < 3\pi/2$. Hence,
\begin{align*}
    \lim_{z \to \infty}\mathcal{D}(z) &= \lim_{z \to \infty} \bigg( \log z +  \log \bigg [ U\bigg(\kappa/2, 1 - \frac{\kappa}{2\gamma} + \frac{\kappa}{2}; \beta z\bigg)^{\frac{2}{\kappa}} \bigg ] \bigg ) \\
    &= \lim_{z \to \infty} \log z + \lim_{z \to \infty} \log \bigg [ (\beta z)^{-\frac{\kappa}{2}\cdot\frac{2}{\kappa}}\bigg ] \\
    &= \lim_{z \to \infty} \left( \log z +  \log z^{-1} \right)  + \lim_{z \to \infty} \log\bigg [ \beta^{-1}\bigg ] \\
    &= -\log \beta.
\end{align*}
Hence, the Shannon transform of the limiting spectral density of $\WW^T \WW$ is given by
\begin{align*}
    \mathcal{V}(x) &= - \frac{2}{\kappa}\log \bigg [ U\bigg(\kappa/2, 1 - \frac{\kappa}{2\gamma} + \frac{\kappa}{2}; \beta x\bigg) \bigg ] -\log (\beta x),
\end{align*}
Finally, note that for $ b < 1$, $U(a,b,z) = \Gamma(1-b)/\Gamma(a - b + 1)$ as $z \to 0$ \citep{DLMF}. Hence,
\begin{align*}
    \lim_{x \to 0} \left(\mathcal{V}_{\bar{\rho}}(x) + \log x \right) = -\frac{2}{\kappa} \log \left[ \frac{\Gamma(\frac{\kappa(1 - \gamma)}{2\gamma})}{\Gamma(\frac{\kappa}{2\gamma})} \right] - \log \beta.
\end{align*}
\end{proof}

\subsection{First and second moments}
The limiting trace of $\WW^T \WW$, under our spectral assumptions, is given by the first moment of \cref{eq:htmp}. We derive this now (as well as the second moment).
\begin{lemma}[First and Second Moment]\label{lemma:derive_moments}
    Denoting $\mu_k$ as the $k^{\text{th}}$ moment of the $\textbf{HTMP}_{\gamma, \kappa, \beta}$ distribution, i.e., $\mu_k = \int_{\real} \lambda^k \bar{\rho}_{\gamma, \kappa, \beta}(d \lambda)$,
    \begin{align*}
        \mu_1 &= \frac{\kappa}{2 \gamma \beta}, \quad\mu_2 = \frac{\kappa}{2 \gamma \beta^2} \cdot \left(\frac{\kappa(1+\gamma)}{2\gamma} +1\right).
    \end{align*}
\end{lemma}

\begin{proof}[Proof of \cref{lemma:derive_moments}]
For probability measures $\nu$ with finite first and second moments, $m_1$ and $m_2$, we have the following expansion of the Stieltjes transform, for argument $z \to -\infty$ \citep{mingo2017free},
\begin{align*}
    S_{\nu}(z) = - z^{-1} - m_1 z^{-2} - m_2 z^{-3} + o(|z|^{-3}).
\end{align*}
Hence, for \cref{eq:htmp}, the first and second moments $m_1$ and $m_2$ of $\bar{\rho}_{\gamma, \kappa, \beta}$ are given by
\begin{align*}
    \mu_1 = \lim_{z \to -\infty} -z^2 S_{\bar{\rho}}(z) - z, && \mu_2 = \lim_{z \to -\infty} -z^3 S_{\bar{\rho}}(z) - z^2 - z\mu_1.
\end{align*}
An application of \cref{lemma:TricomiAsymptotics} then gives the result.
\end{proof}

\subsection{Proof of Theorem 2.}\label[appendix]{sec:appendix:proof_main}
We can now derive the limiting form of \cref{eq:bfe_weights}, by combining \cref{lemma:derive_moments} and \cref{lemma:derive_shannon}; the result is stated in \cref{thm:main}.

\begin{proof}[Proof of \cref{thm:main}]\label{proof:main}
    From \cref{ass:main}, \cref{ass:int_1} and \cref{ass:int_2}, coupled with \cref{lemma:derive_moments} and \cref{lemma:derive_shannon}, we have that
    \begin{align*}
        \lim_{d_{l-1}, d_l \to \infty} \frac{1}{d_{l}} \Tr (\WW^T \WW) &\overset{\mathbb{P}}{\to} \mu_1 + \tau \\
        \lim_{d_{l-1}, d_l \to \infty} \frac{1}{d_{l}} \log \det \WW^T \WW &\overset{\mathbb{P}}{\to} \sV(\tau) + \log \tau.
    \end{align*}
\end{proof}

\section{Proofs for convexity and optimality of free energy}
Having derived the limiting deterministic quantity for \cref{eq:bfe_weights}, under \cref{ass:main}, \cref{ass:int_1}, \cref{ass:int_2}, we now provide various results and proofs on the behaviour of $\sF_{\infty}^{\lambda}$ with respect to $\kappa$ and $\beta$. This section will culminate with the proof of \cref{theorem:first_moment_shannon}. 

\subsection{Optimality of free energy under ratio}
Here, we give various results for the case when we are either controlling $h(\kappa) = \kappa / \beta$, for some given $\kappa$, or optimizing $\kappa$ after $h(\kappa)$ has attained its optima. In the general case ($\tau > 0$), we are able to show that the Shannon transform is strictly increasing in $h$; trivially, the first moment is also strictly increasing in $h$.

\begin{lemma}[Increasing Shannon Transform]\label{lemma:increasing_shannon}
    For fixed $\kappa > 0$, the Shannon transform for \cref{eq:htmp}, $\sV_{\bar{\rho}}$, is strictly increasing in $h$.
\end{lemma}
\begin{proof}
    Under the chain rule, 
    \begin{align*}
        \frac{\partial \sV}{\partial h} &= \frac{\partial \sV}{\partial \beta} \frac{\partial \beta}{\partial h} = - \frac{\beta^2}{\kappa} \frac{\partial \sV}{\partial \beta} \\
        &= -\frac{\beta^2}{\kappa}\left( - \frac{2}{\kappa} \frac{\partial}{\partial \beta}\log \bigg [ U\bigg(\kappa/2, 1 - \frac{\kappa}{2\gamma} + \frac{\kappa}{2}; \beta x\bigg) \bigg ] -\frac{1}{\beta} \right) \\
        &= -\frac{\beta^2}{\kappa}\left( - \frac{2}{\kappa}\left( -\frac{\kappa}{2} \frac{x}{\beta} \beta \frac{U\left(\kappa/2 + 1, 2 - \frac{\kappa}{2\gamma} + \frac{\kappa}{2}; \beta x\right)}{U\left(\kappa/2, 1 - \frac{\kappa}{2\gamma} + \frac{\kappa}{2}; \beta x\right)}\right) - \frac{1}{\beta}\right) \\
        &= -\frac{\beta}{\kappa} \left( x S_{\bar{\rho}_{\gamma, \kappa, \beta}}(-x) - 1 \right).
    \end{align*}
    Now, as $S_{\bar{\rho}_{\gamma, \kappa, \beta}}$ is the Stieltjes transform of a normalized probability distribution (\cref{eq:htmp}), then for $x > 0$, $x S_{\Bar{\rho}_{\gamma, \kappa, \beta}}(-x) - 1 = \int x(\lambda + |x|)^{-1} d\rho(\lambda) - \int d\rho(\lambda) = - \int \lambda(\lambda + |x|)^{-1} d\rho(\lambda) < 0$ for non-degenerate $\rho$. Hence, 
    \begin{align*}
        -\frac{\beta}{\kappa} \left( x S_{\Bar{\rho}_{\gamma, \kappa, \beta}}(-x) - 1 \right) > 0,
    \end{align*}
    and $\sV_{\bar{\rho}}(x)$ is strictly increasing in $h(\kappa) = \kappa / \beta$.
\end{proof}

If we take minimum eigenvalue $\tau \to 0$, then we can also show that for fixed $\kappa > 0$, an optimal ratio $h^*$ exists. To show this, we show strict convexity and coercivity of $\sF_{\infty}^{\lambda}$, and then derive the optimal~$h$.
\begin{lemma}[Strict Convexity, Coercivity and Optimality of $\sF_{\infty}^{\lambda}$]\label{lemma:opt_beta}
    For fixed $\kappa > 0$, and $\tau \to 0$, then $\sF_{\infty}^{\lambda}$ is strictly convex and coercive in $h$. Further, $\sF_{\infty}^{\lambda}$ is minimized at $h^*(\kappa) = 2\gamma/\lambda$; equivalently $\beta^*(\kappa) = \kappa\lambda / 2\gamma$.
\end{lemma}
\begin{proof}
    From \cref{lemma:derive_shannon}, 
    \begin{align}
    \lim_{\tau \to 0} \sF_{\infty}^{\lambda} = \frac{\lambda}{2} \frac{\kappa}{2 \gamma \beta} + \frac{1}{\kappa} \log \left[ \frac{\Gamma(\frac{\kappa(1 - \gamma)}{2\gamma})}{\Gamma(\frac{\kappa}{2\gamma})} \right]  + \frac{1}{2} \log \frac{2 \pi \beta}{\lambda}.
\end{align}
    Strict convexity is immediate by seeing that
    \begin{align*}
        \partial_{hh} \lim_{\tau \to 0} \sF_{\infty}^{\lambda} = \frac{1}{2h^2} > 0.
    \end{align*}
    To see coercivity, note that for $h \to 0$, $\lim_{\tau \to 0}\sF_{\infty}^{\lambda} \sim -\frac{1}{2}\log h$, i.e., $\lim_{h\to 0}\lim_{\tau \to 0}\sF_{\infty}^{\lambda} \to \infty$. Further, for $h \to \infty$, $\lim_{\tau \to 0}\sF_{\infty}^{\lambda} \sim (\lambda / 4\gamma)h - \frac{1}{2}\log h \sim (\lambda / 4\gamma)h$. Hence, $\lim_{h\to \infty}\lim_{\tau \to 0}\sF_{\infty}^{\lambda} \to \infty$. Thus, $\lim_{\tau \to 0}\sF_{\infty}^{\lambda}$ is coercive in $h$. As $\lim_{\tau \to 0}\sF_{\infty}^{\lambda}$ is strictly convex and coercive, a unique minimum $h^{*}$ exists. Then,
    \begin{align*}
        \partial_h \lim_{\tau \to 0} \sF_{\infty}^{\lambda} |_{h = h^{*}} = \frac{\lambda}{4\gamma} - \frac{1}{2h^{*}} = 0, 
    \end{align*}
    giving $h^*(\kappa) = 2\gamma/\lambda$, or equivalently $\beta^*(\kappa) = \kappa\lambda / 2\gamma$.
\end{proof}

Thus, for a fixed $\kappa$, an optimal $\beta$ parameter exists (given as a ratio $h^* = \kappa / \beta^*$). We are now interested in seeing how $\lim_{\tau \to 0}\sF_{\infty}^{\lambda}$ behaves in $\kappa$, for $\beta$ that follows the optimal ratio manifold. As we will see, the behavior is surprisingly simple.

\begin{lemma}\label{lemma:opt_kappa}
    At $\beta^*(\kappa) = \lambda \kappa / 2\gamma$, $\lim_{\tau \to 0}\sF_{\infty}^{\lambda}$ is strictly decreasing in $\kappa$. 
\end{lemma}
\begin{proof}
    We take the derivative of $\lim_{\tau \to 0}\sF_{\infty}^{\lambda}$ with respect to $\kappa$:
    \begin{align*}
    \kappa^2\frac{\partial}{\partial \kappa} \lim_{\tau \to 0}\sF_{\infty}^{\lambda} &= \frac{\kappa(1-\gamma)}{2\gamma} \psi\left(\frac{\kappa(1-\gamma)}{2\gamma} \right) - \frac{\kappa}{2\gamma} \psi\left(\frac{\kappa}{2\gamma} \right)\\
    &-\log \Gamma\left(\kappa (1-\gamma)/2\gamma\right) + \log \Gamma\left(\kappa/2\gamma\right)
    - \frac{\kappa(1-\gamma)}{2\gamma} +\frac{\kappa}{2\gamma}.
    \end{align*}
    Define $Q(x) := x\psi(x) - \log \Gamma(x) - x$. Then, defining $a = \kappa(1-\gamma)/2\gamma, b = \kappa/2\gamma$, where $a < b$, we have that
    \begin{align*}
        \kappa^2\frac{\partial}{\partial \kappa} \lim_{\tau \to 0}\sF_{\infty}^{\lambda} = Q(a) - Q(b).
    \end{align*}
    Hence, to show that $\frac{\partial}{\partial \kappa} \lim_{\tau \to 0}\sF_{\infty}^{\lambda} < 0$, we only need to show that $Q$ is strictly increasing. We take the derivative
    \begin{align*}
        Q'(x) = x\psi_1(x) - 1.
    \end{align*}
    From \citet{qi2009sharp}, we get that $\psi_1(x) > 1/(2x^2) + 1/x > 1/x$ for $x > 0$. Hence, $Q'(x) > 0$, and we have the result. 
\end{proof}

Interestingly, this tells us how to reach the infimum of $\lim_{\tau \to 0}\sF_{\infty}^{\lambda}$. We present the optimal joint infimum in the following corollary.

\begin{corollary}\label{corr:infinite_ratio}
    The joint infimum, $\inf_{\kappa, \beta > 0} \lim_{\tau \to 0}\sF_{\infty}^{\lambda}$ is reached for $\beta = \beta^*(\kappa)$ and $\kappa \to \infty$, i.e., $\kappa, \beta \to \infty$, such that $\kappa / \beta = 2\gamma/\lambda$.
\end{corollary}
\begin{proof}
    For each $\kappa$, \cref{lemma:opt_beta} says that $\arginf_{\beta > 0} \lim_{\tau \to 0}\sF_{\infty}^{\lambda} = \kappa \lambda / 2\gamma$. Along this path, \cref{lemma:opt_kappa} gives us that $\sF_{\infty}^{\lambda}$ is strictly decreasing in $\kappa$. Thus, we have the result.
\end{proof}

\begin{remark}
    From \cref{theorem:first_moment_shannon} and \cref{corr:infinite_ratio}, we see that the infimum of $\lim_{\tau \to 0} \sF_{\infty}^{\lambda}$ occurs for $\kappa, \beta \to \infty$ such that $\kappa / \beta \to 2\gamma / \lambda$. Under this limit, $\textbf{HTMP}_{\gamma, \kappa, \beta} \to \lambda^{-1}\textbf{MP}_\gamma$. Note however that \cref{eq:bfe_weights} was derived from models that interpolate the training data; for DNNs, $\textbf{MP}_\gamma$ corresponds to the spectrum of randomly initialized weights, which generally will \textbf{not} interpolate complex training sets. Hence, $\textbf{MP}_\gamma$ is mostly inaccessible for trained weights (at least for sufficiently complex learning tasks). Instead, for trained weights constrained to the interpolation manifold, \cref{eq:bfe_weights} prefers minimum distributional distance to $\textbf{MP}_\gamma$.
\end{remark}

\subsection{Proof of Theorem 3.}\label[appendix]{sec:appendix:convex_proof}
For \cref{theorem:first_moment_shannon}, we see that this result can be proved by combining \cref{lemma:increasing_shannon}, \cref{lemma:opt_beta} and \cref{lemma:opt_kappa}.

\subsection{Proof of Lemma 1.}\label[appendix]{sec:appendix:opt_ratio}
In the main text, we see that to fit the spectral distributions of trained weight matrices, often one can take $\beta \propto \tau^{-1}$, for proportionality constants that differ between models. Under this assumption, we can derive the optimal ratio. The result is given in \cref{lemma:opt_ratio}, and we now give the proof.

\begin{proof}[Proof of \cref{lemma:opt_ratio}]\label{proof:opt_ratio}
    For $\beta \tau = c$, for some $c > 0$, the Tricomi term in $\sF_{\infty}^\lambda$ no longer relies on $\beta$ or $\tau$. Hence, for some $\kappa$, we can take $\partial_\beta \sF_{\infty}^\lambda = 0$. Rearranging for $h = \kappa / \beta$ gives the result.
\end{proof}

\subsection{Auxiliary lemma}
We require a result on the asymptotic expansion of ratios of Tricomi functions.

\begin{lemma}\label{lemma:TricomiAsymptotics}
For $z \to \infty$ such that $z \in \real$, 
    \begin{align*}
        \frac{U(a+1, b+1, \pm z)}{U(a, b, \pm z)} \sim \pm z^{-1} - (a - b + 1)z^{-2} \pm (a - b + 1)(2a - b + 2)z^{-3} \pm \mathcal{O}(z^{-4}).
    \end{align*}
\end{lemma}
\begin{proof}[Proof of \cref{lemma:TricomiAsymptotics}]
We first give the proof for $R := U(a+1, b+1, z) / U(a, b, z)$. We use the asymptotic expansion of the Tricomi Confluent Function as $z \to \infty$ such that $|\arg z| < \frac{3}{2}\pi$ \citep{DLMF},
    \begin{align*}
        U(a,b,z) \sim z^{-a} \sum_{s=0}^\infty (-1)^s\frac{(a)_s (a - b + 1)_s}{s!}z^{-s},
    \end{align*}
    where $(.)_s$ denotes the rising factorial symbol. Leading $z^{-a}$ terms will cancel, leaving
    \begin{align*}
        R \sim z^{-1} \frac{1 + d_1 z^{-1} + d_2 z^{-2} + \cdots}{1 + c_1 z^{-1} + c_2 z^{-2} + \cdots},
    \end{align*}
    where $c_i = (-1)^i(a)_i (a - b + 1)_i / i!, ~ d_i = (-1)^i(a+1)_i (a - b + 1)_i / i!$. We use a power series expansion of the denominator in terms of $z^{-1}$, using the Geometric series expansion
    \begin{align*}
        (1 + c_1 z^{-1} + c_2 z^{-2} + \cdots)^{-1} = 1 - c_1 z^{-1} + (c_1^2 - c_2)z^{-2} + \mathcal{O}(z^{-3}).
    \end{align*}
    Hence,
    \begin{align*}
        R &\sim z^{-1} (1 + d_1 z^{-1} + d_2 z^{-2} + \cdots)(1 - c_1 z^{-1} + (c_1^2 - c_2)z^{-2} + \cdots) \\
        &\sim z^{-1} + (d_1 - c_1) z^{-2} + ((c_1^2 - c_2) - c_1d_1 + d_2) z^{-3} + \mathcal{O}(z^{-4}) \\
        &\sim z^{-1} - (a - b + 1) z^{-2} + (a - b + 1)(2a - b + 2) z^{-3} + \mathcal{O}(z^{-4}).
    \end{align*}
For the negative case, we note that $U(a,b,z)$ has a branch cut on the real-negative axis. However, we can write $-z = \lim_{\epsilon \to 0} ze^{i(\pi - \epsilon)}$. Substituting this into the above asymptotic expansion
\begin{align*}
    U(a,b,-z) \sim z^{-a}e^{-i\pi a} \sum_{s=0}^\infty \frac{(a)_s (a - b + 1)_s}{s!}z^{-s}.
\end{align*}
The negative factors will cancel in numerator and denominator of $R$, except for a single negative multiplicative factor (arising from exponent of $a + 1$). Hence, we have the result. 
\end{proof}

\section{Further details}
In this section, we provide further theoretical and experimental details which we could not include in the main body.

\subsection{Connection of BFE-W to previous implicit regularization metrics}
To understand why \cref{eq:bfe_weights} correlates so strongly with model quality, we can understand its connection to previous measurements of implicit regularization. 
Firstly, $m^{-1}\Tr (\WW^T \WW)$ is an HT-SR weight metric. This term also provides a connection to the Geometric Complexity \citep{dherin2022neural}; and this term is implicitly regularized after interpolation for gradient flow training \citep{tsilivis2026flavors}.
Secondly, while the log-determinant is less common in the literature, a key term in the IIC is $n^{-1}\log \det \left( \DD \ff~ \DD \ff^T \right)$, where $\DD \ff~ \DD \ff^T$ is the empirical Neural Tangent Kernel \citep{jacot2018neural}. This measures the complexity of $\ff$ through the volume of the distribution of functions induced by parameter Jacobian of $\ff$. The second term in \cref{eq:bfe_weights} instead measures the volume of functions induced by the last-layer weights, i.e., a \textit{functional complexity}, and is the information gain for a Gaussian model with kernel $\WW^T \WW$.

\subsection{Motivation for application of BFE-W for deeper layer weights}\label[appendix]{sec:appendix:deeper_weights}
While \cref{eq:bfe_weights} was derived for final-layer weights, we can motivate the use for weights from deeper layers. Specifically, for deeper weights $\WW_l$, \cref{eq:bfe_weights} corresponds to the Bayes free energy for an auxiliary Gaussian Process with precision matrix $\KK = \lambda \WW_l^T \WW_l$, for artificial outputs of whitened isotropic noise. Once interpolation has occurred, this model measures the inherent spectral geometry of $\WW_l^T \WW_l$ through the trace and log-determinant, and so it is directly measuring the implicit bias inherent in the learned weights. Examining the spectrum of weights to measure implicit model bias is a core aspect of HT-SR that motivated HTMP. Even though this auxiliary energy is not the Bayes free energy for the original dataset and prior-likelihood pair, we see in \cref{fig:free_energy_training} that for the third-layer convolutional weight of a ResNet50 trained on CIFAR-10, \cref{eq:bfe_weights} correlates surprisingly well with test accuracy during training, \textit{once training loss has plateaued}.
We see similar results for ResNet18 and ResNet34 on CIFAR-10 in \cref{fig:free_energy_training_extra}. In \cref{fig:depth_correlation}, we plot the correlation of \cref{eq:bfe_weights} with model quality across all layers, for ResNet50 on CIFAR-100 and all Pythia models. Interestingly, we see that every deeper layer for ResNet50 correlates very well with model quality ($\mathbf{\rho > 0.9}$). For the Pythia models, the Dense weight matrices seem to correlate very well across all layers and model sizes; however, for QueryKey weight matrices, the results are more mixed.

\subsection{Choice of precision parameter \texorpdfstring{$\lambda$}{λ}}\label[appendix]{sec:appendix:lambda}
To select $\lambda$ when comparing model quality with \cref{eq:bfe_weights} across epochs of training $t$, we take two general approaches:
\begin{enumerate}
    \item As detailed in \cref{remark:lambda}, we set $\lambda_t = d_l / ||\WW_{l,t}||_F^2$ at each epoch $t$.
    \item Let $t^*$ denote the epoch when test error is minimized. Then, we set $\lambda_t = d_l / ||\WW_{l,t^*}||_F^2$, i.e., $\lambda$ remains constant.
\end{enumerate}
We find in our experiments that the second option works best for CV models. In fact, taking the first option for $\lambda$ results in a very poor correlation for the CV models. In comparison, both options work for the LLM models. To see why, we note the following. By \cref{lemma:derive_moments}, $\mu_1 = h / 2\gamma$. Thus, $h = 2 \gamma \mu_1$. As $h^* = 2 \gamma / \lambda$, then if we set $\lambda = 1 / \mu_1$, then $h = h^*$. Hence, by selecting $\lambda_t = d_l / ||\WW_{l,t}||_F^2$, we are explicitly assuming that $h$ is already at its optimal value $h^*$. As seen in \cref{fig:free_energy_training_htmp} and \cref{fig:free_energy_training_htmp_extra}, $h$ moving through $h^*$ directly correlates with under- and over-fitting in CV models; taking $\lambda_t = d_l / ||\WW_{l,t}||_F^2$ destroys this behavior. In comparison, we see in \cref{fig:htmp_params} and \cref{fig:htmp_params_extra} that for the LLM models, $h_t \approx h^*_t$ at each epoch. Hence, taking $\lambda_t = d_l / ||\WW_{l,t}||_F^2$ does not alter the correlation. Similar behaviour has arisen in the literature: in \citet{martin2021predicting}, the spectral size $||\WW_{l,t}||_F^2$ correlated strongly with model quality for well-trained networks within a model class, alongside metrics such as the tail exponent $\alpha$. In comparison, \citet{yang2022evaluating} demonstrated that for LLMs, $||\WW_{l,t}||_F^2$ was \textit{anti-correlated} with model quality, suggesting that spectral size plays a differing role for measuring model quality of LLMs versus CV models.

Being able to correctly predict the optimal $\lambda$ value, i.e., $\lambda_t = d_l / ||\WW_{l,t^*}||_F^2$, would enable our framework to accurately measure the amount of overfitting during training, and it is an important avenue for future work. 

\subsection{Relation of optimal HTMP parameters to heavy-tailed results in literature}
We now detail similarities between the optimal scale-based heaviness derived in \cref{theorem:first_moment_shannon} and \cref{lemma:opt_ratio} and those derived in the literature. 

While previous work in the heavy-tailed literature consider mainly the tail behavior \citep{bartlett2020benign, agrawal2022alpha}, the fact that the first moment is proportional to $h$ suggests that a distribution's mean shares importance with tail behavior. 
This also suggests deeper connections between heavy-tailed metrics and Geometric Complexity \citep{dherin2022neural}, discussed previously. 
We see that the optimal heaviness $h^* = 2\gamma / \lambda$ in \cref{theorem:first_moment_shannon} is a function of the aspect ratio $\gamma$; dependence of heavy-tails on the aspect ratio was argued for in \citet{hu2025eigenspectrum}. 
If we consider normalized heaviness $\bar{h} = \lambda \kappa / 2\gamma\beta$ (i.e., $\lambda \mu_1$), we find $\bar{h}^* = 1$. 
This is remarkably similar to the work of \citet{bartlett2020benign}, where the minimum-norm solution for over-parameterized linear regression provides good generalization performance if power-law exponent on the linear feature covariance is equal to $1$. 
A similar finding for DNNs was shown in \citet{agrawal2022alpha}. 
Defining $\bar{h} = \lambda \kappa / \gamma\beta$, optimal heaviness $\bar{h}^* = 2$ places us in the moderately-heavy tailed universal class of HT-SR, which is expected to perform well \citep{martin2021implicit}.
The clear numerical connection between these previous works and ours is surprising and motivating, despite our differing definitions and matrices. 

\section{Experimental details}\label[appendix]{sec:appendix:experimental_details}
\subsection{ResNets during training}
The ResNet models used were each trained for $200$ epochs on the CIFAR-10 / CIFAR-100 dataset. For the optimizer, stochastic gradient descent was used, with momentum $0.9$, batch size $100$, learning rate $0.02$, and weight-decay $0.0001$. No data-augmentation nor dropout was used, so as not to ``over-regularize'' the weight matrices, preventing clear heavy-tailed ESD shapes. We followed the experimental details of \citep{hodgkinson2025models}, and we normalized the input data to have zero mean and unit variance per channel, as well as normalized each input vector $\xx$ to have norm $||\xx|| = \sqrt{d}$, where $d = 3 \times 32 \times 32$. We initialized the networks using the LeCun normal initialization. Each network was randomly initialized and trained $3$ times. Checkpoints were taken every $5$ epochs. These models were run on an H$100$ GPU.

\subsection{Fitting the HTMP} \label{appendix:experiment:fit_htmp}
To fit the $\textbf{HTMP}_{\kappa, \gamma, \beta}$ distribution to the ESD for $\spec(\WW_l^T \WW_l - \tau \eye)$, for trained weight matrix $\WW_l \in \real^{d_{l-1} \times d_l}$, we firstly fix $\gamma = d_l / d_{l-1}$. For some weight matrices, we trim the top and bottom $5$ eigenvalues, as these extreme eigenvalues can make fitting more difficult. 

Let $\omega_1, \dots, \omega_{d_l}$ be the eigenvalues of $\WW_l^T \WW_l$. Using the HTMP CDF from \cref{lemma:htmp_cdf}, denoted $F_{\kappa, \gamma, \beta}$, we minimize the two-sided Kolmogorov-Smirnov distance,
\begin{align*}
    D_{KS}(\kappa, \beta) = \max_i \left\{ \left | F_{\kappa, \gamma, \beta}(\omega_i - \tau) - \frac{i}{N} \right |,  \left | F_{\kappa, \gamma, \beta}(\omega_i - \tau) - \frac{i-1}{N} \right |\right \}.
\end{align*}
In many cases, we can directly match the empirical and theoretical first moment, using the first HTMP moment from \cref{lemma:derive_moments}, and set
\begin{align*}
    \beta(\kappa) = \frac{\kappa d_l}{2 \gamma \left(||\WW_{l}||_F^2 - d_l \tau\right)}.
\end{align*}
Then, 
\begin{align*}
    \kappa^* = \argmin D_{KS}(\kappa, \beta(\kappa)).
\end{align*}

To compute $F_{\kappa, \gamma, \beta}(\omega_i - \tau)$, then if $1 - \kappa / 2\gamma + \kappa / 2$ is \textbf{not} an integer, then we can take $\omega_1, \dots, \omega_{d_l}$, and compute the principal argument 
\begin{align*}
    \theta_i = \text{atan2}\left(\mathcal{I}(a,b,\beta(\omega_i - \tau)), \mathcal{R}(a,b,\beta(\omega_i - \tau))\right),
\end{align*}
where $a = \kappa / 2, b = 1 - \kappa / 2\gamma + \kappa / 2$, and $\mathcal{I}$ and $\mathcal{R}$ are the imaginary and real parts of $\lim_{y \downarrow 0} U(a,b,-x-iy)$, given in \cref{lemma:cdf_noninteger}. Then, as $\theta_i \in (-\pi, \pi]$, we can use an unwrapping operation (implemented numerically in common computing packages), that takes a set $\theta_1, \dots, \theta_{d_l}$, and adds integer multiples of $2\pi$ to the values in this set whenever the principal argument has passed through the boundary cut. This will return a set $\Theta_1, \dots, \Theta_{d_l} \in [0,\infty)$, such that $F_{\kappa, \gamma, \beta}(\omega_i - \tau) = 2 \Theta_i / \pi \kappa$. 

In the case when $1 - \kappa / 2\gamma + \kappa / 2$ \textbf{is} an integer, one must instead compute $\lim_{y \downarrow 0} U\left( \kappa / 2, 1 - \kappa/2 \gamma + \kappa/2; -\beta x - i \beta y\right)$ using complex numerical functions (such as the implementation of $U(a,b,z)$ in the \textit{mpmath} package), in order to compute the CDF. This process is numerically much slower.

\section{Further experimental results}\label[appendix]{sec:appendix:further_experiments}
In this section, we will give further experimental details, that did not fit in the main body. 

\subsection{Permuted label experiment}
We now demonstrate that even on a task where no feature learning can occur, \cref{eq:bfe_weights} closely follows model quality; further, the HTMP analysis of \cref{theorem:first_moment_shannon} and \cref{lemma:opt_ratio} can still predict the optimal model over training. For this task, we train ResNet18 on CIFAR-10 where all training labels have been randomly permuted. Hence, training inputs $\xx_i$ are now independent of training outputs $\yy_i$. For this task, the Bayes optimal accuracy is $10\%$, i.e., random guessing for $10$ classes. Instead of accuracy, the appropriate metric to consider is cross-entropy on the test dataset. An optimal model will produce a prediction of $\text{softmax}(\ff(\xx_i^*; \WW, \phi)) = [0.1]^{10}$ for each input $\xx_i^*$ in a test dataset. 

In \cref{fig:free_energy_training_permuted} (top left) we see that HTMP still provides an excellent fit to the weight matrices in this case. Here, the weight matrix used for analysis is the third-layer convolutional weight matrix. We plot test error, empirical and asymptotic free energy against training epochs in \cref{fig:free_energy_training_permuted} (top right), for $\lambda^* = 4.72$, where we see a very strong correlation ($\mathbf{\rho = 0.819}$). We mark with a dotted line the epoch when training accuracy reaches $99\%$, which we loosely interpret as \textit{interpolation}. While the absolute scale of free energy does not differentiate between the untrained and trained models, once we constrain to interpolating epochs, the free energy very closely predicts test error. 

Importantly, we also see in \cref{fig:free_energy_training_permuted} (bottom left) that after interpolating on the permuted training set, the optimal model corresponds to the one whose HTMP ratio $h = \kappa / \beta$ is closest to $h^*$ (in \cref{fig:free_energy_training_permuted}(bottom right) we see that the assumptions of \cref{lemma:opt_ratio} are valid). Thus, even in this artificial learning task, our framework allows us to find what spectral properties the model is trying to optimize. 

\begin{figure}[t]
    \centering
    \begin{subfigure}[b]{0.49\linewidth}
        \includegraphics[width=\linewidth]{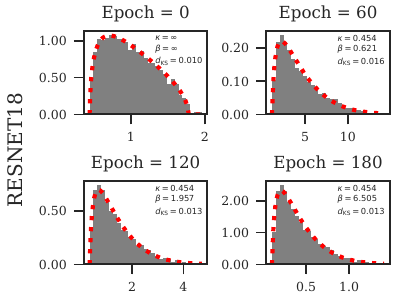}
    \end{subfigure}
    \begin{subfigure}[b]{0.49\linewidth}
        \includegraphics[width=\linewidth]{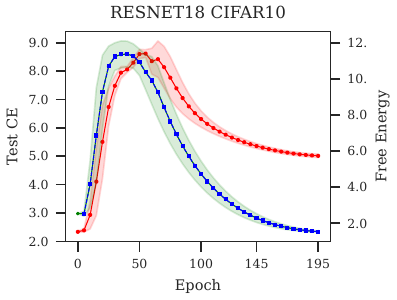}
    \end{subfigure}
    \begin{subfigure}[b]{0.49\linewidth}
        \includegraphics[width=\linewidth]{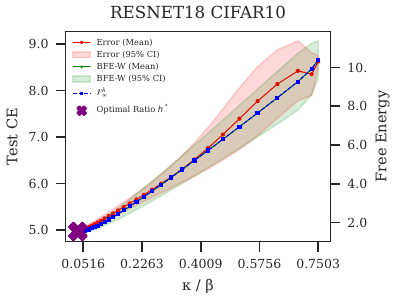}
    \end{subfigure}
    \begin{subfigure}[b]{0.49\linewidth}
        \includegraphics[width=\linewidth]{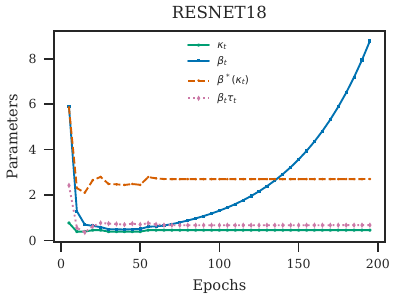}
    \end{subfigure}
    \caption{For ResNet18, trained on permuted CIFAR-10 dataset, we have the (top left) ESD (gray) and fitted $\textbf{HTMP}_{\kappa, \gamma, \beta, \tau}$ distribution (red dotted lines), (top right) test CE, \cref{eq:bfe_weights} and $\sF_{\infty}^{\lambda}$, vs. epochs, where the epoch at which training accuracy reaches $0.99$ is indicated with a vertical dotted line, (bottom left) test CE, \cref{eq:bfe_weights} and $\sF_{\infty}^{\lambda}$, vs. $h = \kappa / \beta$ for \textit{interpolating} epochs, and (bottom right) fitted HTMP parameters. Even for a task where there is no feature learning, our framework holds surprisingly well. Here, the weight matrix used for analysis is the third-layer convolutional weight.}
    \label{fig:free_energy_training_permuted}
\end{figure}

\subsection{Extra ResNet and Pythia models}
Here, we provide further experimental evidence for the applicability of \cref{theorem:first_moment_shannon} and \cref{lemma:opt_ratio} for analyzing training dynamics of large DNNs. Specifically, we repeat the analysis for ResNet18 and ResNet34 trained on CIFAR-10, and for the Pythia $2.8$B and Pythia $6.9$B models. We see in \cref{fig:free_energy_training_extra}, \cref{fig:free_energy_training_htmp_extra} and \cref{fig:htmp_params_extra} that for CV models, overfitting is a function of $h$ moving through $h^*$, whereas for the LLM models, $h \approx h^*$ during training, and so $\kappa$ increasing results in an improved model fit. 

\begin{figure}[t]
    \centering
    \begin{subfigure}[b]{0.49\linewidth}
        \includegraphics[width=\linewidth]{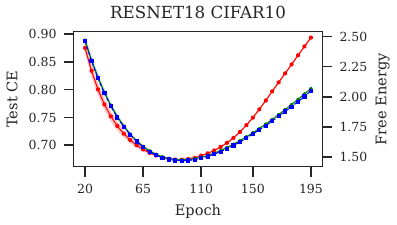}
    \end{subfigure}
    \begin{subfigure}[b]{0.49\linewidth}
        \includegraphics[width=\linewidth]{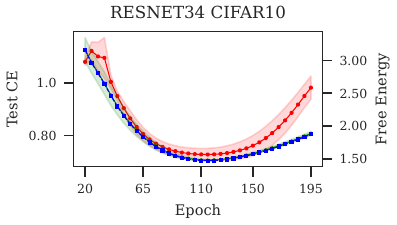}
    \end{subfigure}
    \begin{subfigure}[b]{0.49\linewidth}
        \includegraphics[width=\linewidth]{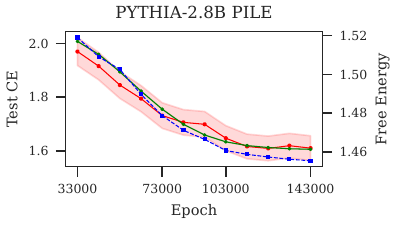}
    \end{subfigure}
    \begin{subfigure}[b]{0.49\linewidth}
        \includegraphics[width=\linewidth]{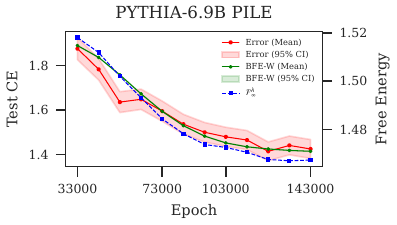}
    \end{subfigure}
    \caption{Test CE, \cref{eq:bfe_weights} and $\sF_{\infty}^{\lambda}$, vs. epochs. Model quality during training is highly correlated with BFE. Further, $\sF_{\infty}^{\lambda}$ closely approximates \cref{eq:bfe_weights}. Extra experimental results to accompany \cref{fig:free_energy_training}.}
    \label{fig:free_energy_training_extra}
\end{figure}

\begin{figure}[t]
    \centering
    \begin{subfigure}[b]{0.49\linewidth}
        \includegraphics[width=\linewidth]{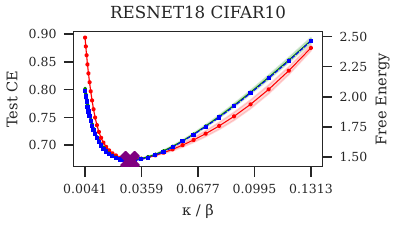}
    \end{subfigure}
    \begin{subfigure}[b]{0.49\linewidth}
        \includegraphics[width=\linewidth]{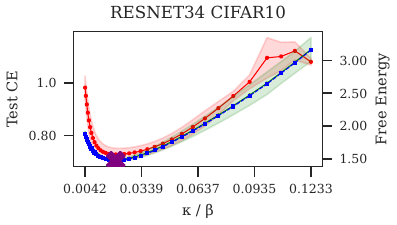}
    \end{subfigure}
    \begin{subfigure}[b]{0.49\linewidth}
        \includegraphics[width=\linewidth]{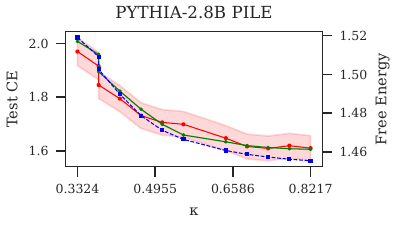}
    \end{subfigure}
    \begin{subfigure}[b]{0.49\linewidth}
        \includegraphics[width=\linewidth]{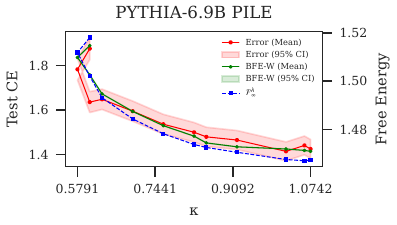}
    \end{subfigure}
    \caption{Test error, \cref{eq:bfe_weights} and $\sF_{\infty}^{\lambda}$ vs. fitted $\kappa / \beta$ and $\kappa$. For CV, $h^*$ (purple cross) aligns with minimum error. For LLMs, $h^*$ is attained early, and increasing $\kappa$ lowers test error. Extra experimental results to accompany \cref{fig:free_energy_training_htmp}.}
    \label{fig:free_energy_training_htmp_extra}
\end{figure}

\begin{figure}
    \centering
    \includegraphics[width=\linewidth]{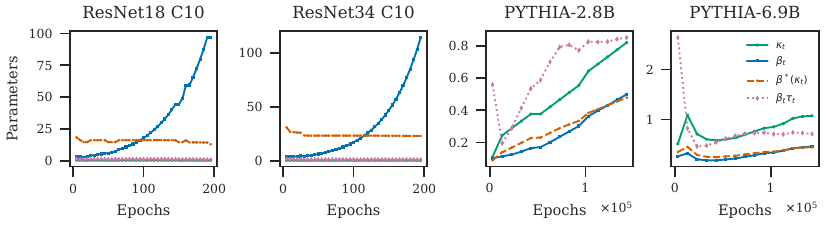}
    \caption{Fitted $\textbf{HTMP}_{\gamma, \kappa, \beta, \tau}$ parameters at each epoch. Extra experimental results to accompany \cref{fig:htmp_params}.}
    \label{fig:htmp_params_extra}
\end{figure}

\subsection{ESDs}\label[appendix]{sec:appendix:esds}
Here, we provide ESDs for ResNet50 trained on CIFAR-10 and Pythia-$1B$ trained on Pile, to accompany the results given in \cref{fig:free_energy_training}. Note that the ESD for the other two models in \cref{fig:free_energy_training} were given in \cref{fig:esd_training}.

\begin{figure}[t]
    \centering
    \begin{subfigure}[b]{1\textwidth}
        \includegraphics[width=\textwidth]{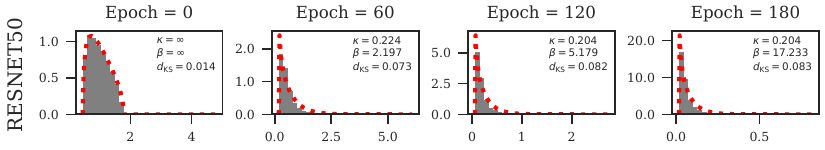}
    \end{subfigure}
    \begin{subfigure}[b]{1\textwidth}
        \includegraphics[width=\textwidth]{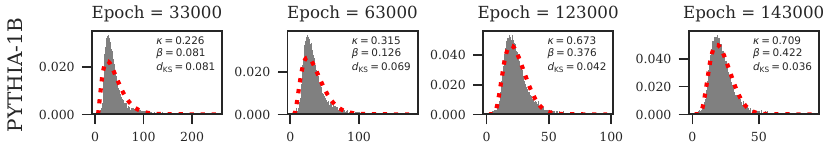}
    \end{subfigure}
    \caption{The ESD (gray) and fitted $\textbf{HTMP}_{\kappa, \gamma, \beta, \tau}$ distribution (red dotted lines), for (top) the third-layer convolutional weight of ResNet50 trained on CIFAR-10, and (bottom) the \textit{EmbedOut} weight of Pythia-$1B$ trained on Pile.}
    \label{fig:esd_training_2}
\end{figure}

\subsection{Correlation of free energy with model quality across depth}
Though we derived \cref{eq:bfe_weights} for the final-layer weights, a natural question is whether the correlation between free energy and model quality also holds for weights from deeper layers. We have already seen evidence of this in \cref{fig:free_energy_training} for ResNet50 on CIFAR-10, where the free energy evaluated on the third-layer convolutional weight matrix correlated very strongly with model quality. We also provide theoretical motivation in \cref{sec:appendix:deeper_weights}. Here, we provide extensive evaluation of correlation of free energy with model quality for deeper weights, across all weights of ResNet50 on CIFAR-100, and for Dense and QueryKey matrices at every layer for all Pythia models. We measure correlation from epoch $20$ and epoch $43000$ respectively, so that training loss effects can be removed (remember \cref{eq:bfe_weights} is derived for interpolating models). We provide the results in \cref{fig:depth_correlation}. Interestingly, \cref{eq:bfe_weights} for \textbf{every} weight for ResNet50 on CIFAR-100 correlates very strongly with model quality. We also see that $\lambda^* = d_l / ||\WW_{l,t^*}||_F^2$ is approximately constant across layers, and that distinct patterns emerge based on convolutional ordering (i.e., Conv 1 vs. Conv 2 weights seem to roughly double $\lambda^*$, and similarly for Conv 2. vs Conv. 3). For the Pythia models, the results are not as clear. For some models (e.g., Pythia $2.8$B), both QueryKey and Dense weight matrices at every layer show excellent correlation between \cref{eq:bfe_weights} and model quality; in contrast Pythia-$12$B QueryKey weights show a very strong \textit{anti-correlation} with model quality across most layers. This figure demonstrates that many weights can be used to evaluate \cref{eq:bfe_weights}, though some care needs to be taken. Further, the relationship of $\lambda^*$ across layers and convolutional weights suggests that one may be able to derive the value of $\lambda^*$ based off of model parameters; this would allow for a priori prediction of when overfitting would occur.

\begin{figure}[t]
    \centering
    \begin{subfigure}[b]{1\textwidth}
        \includegraphics[width=\textwidth]{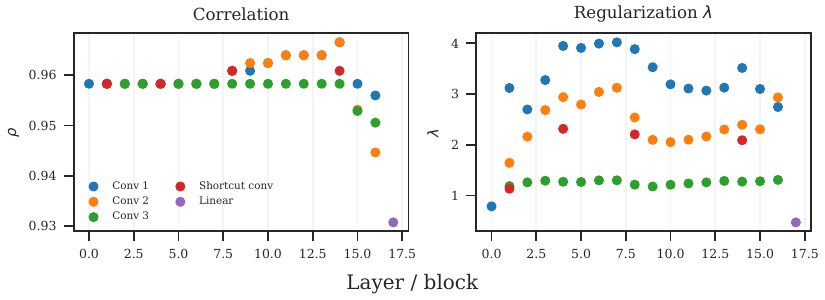}
    \end{subfigure}
    \begin{subfigure}[b]{1\textwidth}
        \includegraphics[width=\textwidth]{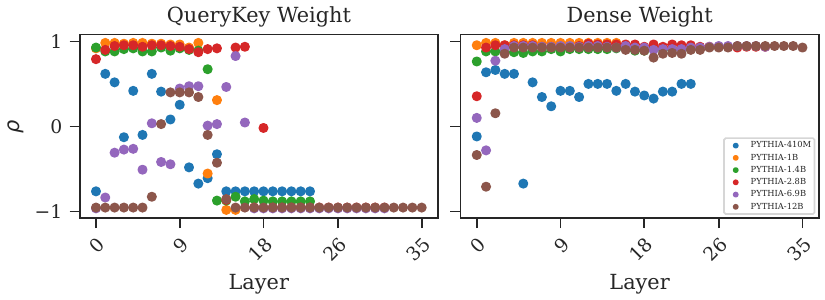}
    \end{subfigure}
    \caption{Correlation of \cref{eq:bfe_weights} with model quality for deeper weights, across (top) all weights of ResNet50 on CIFAR-100, and (bottom) for Dense and QueryKey matrices at every layer for all Pythia models. For the Pythia models, model quality is measured by log perplexity for the LAMBADA test dataset. We also provide the $\lambda^* = d_l / ||\WW_{l,t^*}||_F^2$ value for all ResNet50 weights. We see that a significant proportion of deeper weights show a very strong correlation with model quality.}
    \label{fig:depth_correlation}
\end{figure}

\section{Further empirical results for free energy}

\subsection{Comparison to HT-SR metrics}
\label[appendix]{sec:appendix:WW}

We now directly compare \cref{eq:bfe_weights} to the HT-SR metrics of \citet{martin2021predicting}, as defined in \cref{sec:background}. In this work, the metrics were computed using the WeightWatchers (WW) package \citep{WeightWatcher}. To ensure fair comparison, we use the legacy version of WW (version $0.2.7$) which is used in \citet{martin2021predicting}. We compare \cref{eq:bfe_weights} (primal) against $\|\WW\|_F^2$, $\|\WW\|_{\infty}^2$, $\hat{\alpha}$, $\|\WW\|_{2\alpha}^{2\alpha}$, $\alpha$ (PL exponent), and $\lambda$ (ET-PL truncation).

\begin{table}[t]
\caption{Spearman $\rho$ correlation of model quality metrics with top-5 ImageNet test accuracy for \textbf{Legacy} models. Combined pools all models in this type. The $\alpha$ and $\lambda$ are the fitted PL and ET-PL parameters. }
\label{table:legacy_ww}
\centering
\begingroup
\setlength{\tabcolsep}{2pt}
\begin{tabularx}{\textwidth}{@{}l*{7}{>{\centering\arraybackslash}X}@{}}
\toprule
Model & BFE-W & $\|\WW\|_F^2$ & $\|\WW\|_{\infty}^2$ & $\hat{\alpha}$ & $\|\WW\|_{2\alpha}^{2\alpha}$ & $\alpha$ & $\lambda$ \\
\midrule
DenseNet & $-0.80$ & $-1.00$ & $-1.00$ & $-1.00$ & $-1.00$ & $0.00$ & $0.40$ \\
ResNet & $-0.90$ & $-1.00$ & $-1.00$ & $-1.00$ & $-1.00$ & $-0.70$ & $1.00$ \\
VGG & $0.81$ & $-0.90$ & $-0.98$ & $-0.98$ & $-0.98$ & $0.95$ & $0.95$ \\
\midrule
Combined & $0.48$ & $-0.88$ & $-0.89$ & $-0.87$ & $-0.87$ & $0.54$ & $0.86$ \\
\bottomrule
\end{tabularx}
\endgroup
\end{table}

\begin{table}[t]
\caption{Spearman $\rho$ correlation of model quality metrics with top-5 ImageNet test accuracy for \textbf{CV} (non-dropout) models. Combined pools all models in this type. The $\alpha$ and $\lambda$ are the fitted PL and ET-PL parameters. }
\label{table:resnettype_ww}
\centering
\begingroup
\setlength{\tabcolsep}{2pt}
\begin{tabularx}{\textwidth}{@{}l*{7}{>{\centering\arraybackslash}X}@{}}
\toprule
Model & BFE-W & $\|\WW\|_F^2$ & $\|\WW\|_{\infty}^2$ & $\hat{\alpha}$ & $\|\WW\|_{2\alpha}^{2\alpha}$ & $\alpha$ & $\lambda$ \\
\midrule
ConvNeXt & $-0.80$ & $1.00$ & $1.00$ & $1.00$ & $1.00$ & $0.40$ & $-0.80$ \\
DenseNet & $-0.80$ & $-0.40$ & $-0.40$ & $-0.40$ & $-0.40$ & $-0.20$ & $0.40$ \\
Inception & \textemdash & \textemdash & \textemdash & \textemdash & \textemdash & \textemdash & \textemdash \\
RegNet & $-0.92$ & $0.30$ & $0.15$ & $0.33$ & $0.35$ & $-0.30$ & $-0.39$ \\
ResNeXt & $0.50$ & $-0.50$ & $-1.00$ & $-1.00$ & $-0.50$ & $1.00$ & $1.00$ \\
ResNet & $-1.00$ & $0.50$ & $-0.20$ & $-0.20$ & $0.10$ & $0.40$ & $0.50$ \\
ShuffleNet & $-0.80$ & $-0.20$ & $0.80$ & $-0.20$ & $-0.20$ & $0.80$ & $0.40$ \\
Wide ResNet & $-1.00$ & $-1.00$ & $-1.00$ & $-1.00$ & $-1.00$ & $1.00$ & $1.00$ \\
\midrule
Combined & $-0.71$ & $0.54$ & $0.35$ & $0.53$ & $0.51$ & $0.58$ & $-0.52$ \\
\bottomrule
\end{tabularx}
\endgroup
\end{table}

Firstly, in \cref{table:legacy_ww}, we compare \cref{eq:bfe_weights} against the HT-SR metrics for the DNNs featured in \citet{martin2021predicting}, i.e., ResNet, VGG and DenseNet. We compare using the legacy model weight checkpoints used in \citet{martin2021predicting}. The reported correlations are Spearmans $\rho$ correlation. We see that, as expected, both within-group correlation and across-group correlation is approximately similar (and very strong) across \cref{eq:bfe_weights} and the HT-SR metrics, with the exception of VGG; as stated in \cref{sec:free_energy_experiment}, VGG employs \textit{dropout} before the final-layer weights, and hence \cref{eq:bfe_weights} performs very poorly on these models. Due to this, \cref{eq:bfe_weights} has the \textit{opposite} correlation with model quality across these groups. Further, we see that $\alpha$ and $\lambda$ are poor predictors of model quality. 

We now compare \cref{eq:bfe_weights} for all \textbf{torchvision} models that are not \textit{dropout} or \textit{vision transformer} models. We use the most updated model weights as provided by \textit{torchvision}, as of May 2026. The results are displayed in \cref{table:resnettype_ww}. Here, while \cref{eq:bfe_weights} correlates very strongly, we see that the HT-SR metrics now do not predict model quality across model classes. There is still some correlation within-group, i.e., for ResNeXt and Wide ResNet, though we observe that \cref{eq:bfe_weights} is the most consistent predictor. Of particular note is that these metrics no longer predict test accuracy for ResNets. As the current available weights have been finely optimized for performance, it is perhaps the case that the weights have become over-regularized in some sense, and so one must develop further metrics (e.g., \cref{eq:bfe_weights}) to evaluate test performance. We again see that $\alpha$ and $\lambda$ are poor predictors of model quality.

\subsection{When last-layer free energy does not predict model quality}\label[appendix]{sec:appendix:bad_free_energy}
As detailed in \cref{sec:free_energy_experiment}, certain classes of models do not show strong correlation between \cref{eq:bfe_weights} and model quality for CV.
\begin{description}
    \item[Dropout.] Models that use dropout directly before the final-layer weights, such as VGG \citep{simonyan2014very}, EfficientNet \citep{tan2021efficientnetv2}, and MNASNet \citep{tan2019mnasnet}, do not show good correlation. This could be because the addition of dropout violates the \textit{features} and \textit{weights} decomposition of a neural network, or perhaps because dropout prevents complex correlations in the final-layer weights from occurring.
    \item[Vision Transformers.] Vision transformers, e.g., ViT \citep{dosovitskiy2020image} and Swin \citep{liu2021swin, liu2022swin}, also do not show good correlation. We conjecture that a clear decomposition of vision transformers into \textit{features} (final activations) and \textit{weight} (final layer) does not hold (though it does for language transformers). 
\end{description}

\begin{table}[t]
\caption{Spearman $\rho$ correlation of model quality metrics with top-5 ImageNet test accuracy for \textbf{Dropout} models. Combined pools all models in this type. The $\alpha$ and $\lambda$ are the fitted PL and ET-PL parameters. }
\label{table:dropout_ww}
\centering
\begingroup
\setlength{\tabcolsep}{2pt}
\begin{tabularx}{\textwidth}{@{}l*{7}{>{\centering\arraybackslash}X}@{}}
\toprule
Model & BFE-W & $\|\WW\|_F^2$ & $\|\WW\|_{\infty}^2$ & $\hat{\alpha}$ & $\|\WW\|_{2\alpha}^{2\alpha}$ & $\alpha$ & $\lambda$ \\
\midrule
EfficientNet & $0.54$ & $-0.83$ & $-0.84$ & $-0.81$ & $-0.83$ & $-0.95$ & $0.69$ \\
MNASNet & $0.80$ & $-0.80$ & $-0.80$ & $-0.80$ & $-0.80$ & $0.20$ & $0.40$ \\
MobileNet & $-1.00$ & $-0.50$ & $-1.00$ & $-0.50$ & $-0.50$ & $-1.00$ & $0.50$ \\
VGG & $0.81$ & $-0.90$ & $-0.98$ & $-0.93$ & $-0.93$ & $0.14$ & $0.95$ \\
\midrule
Combined & $0.00$ & $-0.42$ & $-0.66$ & $-0.45$ & $-0.42$ & $-0.08$ & $0.73$ \\
\bottomrule
\end{tabularx}
\endgroup
\end{table}

\begin{table}[t]
\caption{Spearman $\rho$ correlation of model quality metrics with top-5 ImageNet test accuracy for \textbf{Vision Transformer} models. Combined pools all models in this type. The $\alpha$ and $\lambda$ are the fitted PL and ET-PL parameters. }
\label{table:vit_ww}
\centering
\begingroup
\setlength{\tabcolsep}{2pt}
\begin{tabularx}{\textwidth}{@{}l*{7}{>{\centering\arraybackslash}X}@{}}
\toprule
Model & BFE-W & $\|\WW\|_F^2$ & $\|\WW\|_{\infty}^2$ & $\hat{\alpha}$ & $\|\WW\|_{2\alpha}^{2\alpha}$ & $\alpha$ & $\lambda$ \\
\midrule
Swin & $-0.83$ & $0.66$ & $0.94$ & $0.89$ & $0.89$ & $0.83$ & $-1.00$ \\
ViT & $0.20$ & $0.20$ & $-0.50$ & $-0.30$ & $-0.30$ & $0.30$ & $0.30$ \\
\midrule
Combined & $0.14$ & $0.17$ & $-0.31$ & $0.07$ & $0.03$ & $0.77$ & $-0.09$ \\
\bottomrule
\end{tabularx}
\endgroup
\end{table}

We provide the correlations for \cref{eq:bfe_weights}, as well as the HT-SR metrics, in \cref{table:dropout_ww} and \cref{table:vit_ww}.
For \textbf{Vision Transformers}, the HT-SR metrics are also poor predictors of model quality. However, though correlation is only mild ($-0.42 \geq \rho \geq -0.66$), the HT-SR metrics are much stronger predictors than \cref{eq:bfe_weights} for the Dropout models. Interestingly, $\alpha$ and $\lambda$ still do not show consistent correlation.

\clearpage

\end{document}

%% file: Format/packages.tex
\usepackage{bm}
\usepackage{amsfonts}
\usepackage{amssymb}
\usepackage{algorithmic,algorithm}
\usepackage{caption}
\usepackage{subcaption}
\usepackage{booktabs}
\usepackage{tabularx}
\usepackage{xcolor}
\usepackage{pifont}
\usepackage{hyperref}
\usepackage{amsmath}
\usepackage{wrapfig}

\usepackage{natbib}

\usepackage{tcolorbox}

\usepackage{amssymb}
\usepackage{mathtools}
\usepackage{amsthm}

\crefname{equation}{}{}
\crefname{figure}{Figure}{Figures}
\creflabelformat{equation}{\textup{(#2#1#3)}}
\crefname{assumption}{Assumption}{Assumptions}
\crefname{condition}{Condition}{Conditions}
\crefname{condition}{Condition}{Conditions}

\usepackage{arydshln}
\usepackage{enumitem}
\setlist[enumerate]{leftmargin=1.3em, itemindent=0em, itemsep=0em, topsep=0em, label = {\bfseries \arabic*.}} 


%% file: Format/newcmnds.tex
\newif\ifsolution 
\newif\iflongversion
\longversionfalse
\newif\ifgrad

\usepackage{pifont}
\usepackage{accents}

\usepackage{xspace}

\usepackage{accents}

\usepackage{stackengine}
\stackMath
\newcommand\tsup[2][2]{%
	\def\useanchorwidth{T}%
	\ifnum#1>1%
	\stackon[-.5pt]{\tsup[\numexpr#1-1\relax]{#2}}{\scriptscriptstyle\sim}%
	\else%
	\stackon[.5pt]{#2}{\scriptscriptstyle\sim}%
	\fi%
}
\makeatletter
\newcommand{\longdash}[1][2em]{%
	\makebox[#1]{$\m@th\smash-\mkern-7mu\cleaders\hbox{$\mkern-2mu\smash-\mkern-2mu$}\hfill\mkern-7mu\smash-$}}
\makeatother
\newcommand{\omitskip}{\kern-\arraycolsep}

\newcommand{\real}{\mathbb{R}}
\newcommand{\complex}{\mathbb{C}}

\newcommand{\integer}{\mathbb{Z}}
\DeclareMathOperator*{\argmin}{arg\,min}

\DeclareMathOperator*{\arginf}{arg\,inf}
\DeclareMathOperator*{\smin}{smin}
\DeclareMathOperator*{\argsmin}{argsmin}

\makeatletter
\newcommand*{\transpose}{%
	{\mathpalette\@transpose{}}%
}
\newcommand*{\@transpose}[2]{%
	\raisebox{\depth}{$\m@th#1\intercal$}%
}
\makeatother

\newcommand{\sD}{\mathcal{D}}

\newcommand{\sF}{\mathcal{F}}

\newcommand{\sL}{\mathcal{L}}
\newcommand{\sM}{\mathcal{M}}
\newcommand{\sN}{\mathcal{N}}
\newcommand{\sP}{\mathcal{P}}

\newcommand{\sV}{\mathcal{V}}

\newcommand{\sX}{\mathcal{X}}
\newcommand{\sY}{\mathcal{Y}}
\newcommand{\sZ}{\mathcal{Z}}

\renewcommand {\AA}  { {\mathbf{A}} }

\newcommand {\RR}  { {\mathbf{R}} }
\newcommand {\BB}  { {\mathbf{B}} }

\newcommand {\DD}  { {\mathbf{D}} }

\newcommand {\KK}  { {\mathbf{K}} }

\newcommand {\VV}  { {\mathbf{V}} }
\newcommand {\WW}  { {\mathbf{W}} }
\newcommand {\UU}  { {\mathbf{U}} }

\newcommand {\XX}  { {\mathbf{X}} }
\newcommand {\YY}  { {\mathbf{Y}} }

\newcommand{\eye}{\text{I}}

\newcommand {\ff}  { {\bf f} }

\newcommand {\yy}  { {\bf y} }

\newcommand {\xx}  { {\bf x} }

\newcommand {\btheta} {\bm \theta}

\newcommand {\rank}  { {\textnormal{Rank}} }

\newcommand {\range}  { {\textnormal{Range}} }

\newcommand {\pdet}  { {\textnormal{pdet}} }
\newcommand {\Arg}  { {\textnormal{Arg}} }
\newcommand {\Tr}  { {\textnormal{Tr}} }
\newcommand {\spec}  { {\textnormal{spec}} }

\newcommand {\supp}  { {\textnormal{supp}} }

\definecolor{forestgreen}{rgb}{0.13, 0.55, 0.13}

\definecolor{amber}{rgb}{1.0, 0.75, 0.0}

\definecolor{bananayellow}{rgb}{.8, 0.6, 0}

\newcounter{comment}

\usepackage{amsthm}
\usepackage[framemethod=TikZ]{mdframed}

\newtheorem*{example*}{Example}

\newtheorem{theorem}{Theorem}
\newtheorem*{theorem_nonumber}{Theorem}
\newtheorem{lemma}{Lemma}

\newtheorem{assumption}{Assumption}
\newtheorem{definition}{Definition}
\newtheorem{corollary}{Corollary}
\theoremstyle{definition}
\newtheorem{remark}{Remark}

\usepackage{tikz}
\usepackage{xparse}

\NewDocumentCommand\DownArrow{O{2.0ex} O{black}}{%
	\mathrel{\tikz[baseline] \draw [<-, line width=0.5pt, #2] (0,0) -- ++(0,#1);}
}

\usepackage{listings} 

\definecolor{mygreen}{rgb}{0,0.6,0}
\definecolor{mygray}{rgb}{0.5,0.5,0.5}
\definecolor{mymauve}{rgb}{0.58,0,0.82}
\definecolor{codegreen}{rgb}{0,0.6,0}
\definecolor{codegray}{rgb}{0.5,0.5,0.5}
\definecolor{codepurple}{rgb}{0.58,0,0.82}
\definecolor{backcolour}{rgb}{0.95,0.95,0.92}

\lstdefinestyle{mystyle}{
	backgroundcolor=\color{backcolour},   
	commentstyle=\color{codegreen},
	keywordstyle=\color{magenta},
	numberstyle=\tiny\color{codegray},
	stringstyle=\color{codepurple},
	basicstyle=\ttfamily\footnotesize,
	breakatwhitespace=false,         
	breaklines=true,                 
	captionpos=b,                    
	keepspaces=true,                 
	numbers=left,                    
	numbersep=5pt,                  
	showspaces=false,                
	showstringspaces=false,
	showtabs=false,                  
	tabsize=2
}
\newcommand {\Ex} { {\mathbb E} }

\usepackage{dsfont}

%% file: Bib/biblio.bib
@article{fan2020spectra,
  title={Spectra of the Conjugate Kernel and Neural Tangent Kernel for Linear-Width Neural Networks},
  author={Fan, Zhou and Wang, Zhichao},
  journal={Advances in Neural Information Processing Systems},
  volume={33},
  pages={7710--7721},
  year={2020}
}

@book{couillet2022random,
  title={Random Matrix Methods for Machine Learning},
  author={Couillet, Romain and Liao, Zhenyu},
  year={2022},
  publisher={Cambridge University Press}
}

@inproceedings{hodgkinson2023monotonicity,
  title={Monotonicity and Double Descent in Uncertainty Estimation with {Gaussian} Processes},
  author={Hodgkinson, Liam and van der Heide, Chris and Roosta, Fred and Mahoney, Michael W},
  booktitle={International Conference on Machine Learning},
  pages={13085--13117},
  year={2023},
  organization={PMLR}
}

@article{fong2020marginal,
  title={On the marginal likelihood and cross-validation},
  author={Fong, Edwin and Holmes, Chris C},
  journal={Biometrika},
  volume={107},
  number={2},
  pages={489--496},
  year={2020},
  publisher={Oxford University Press}
}

@article{germain2016pac,
  title={{PAC-Bayesian} theory meets {Bayesian} inference},
  author={Germain, Pascal and Bach, Francis and Lacoste, Alexandre and Lacoste-Julien, Simon},
  journal={Advances in Neural Information Processing Systems},
  volume={29},
  year={2016}
}

@article{jacot2018neural,
  title={Neural tangent kernel: Convergence and generalization in neural networks},
  author={Jacot, Arthur and Gabriel, Franck and Hongler, Cl{\'e}ment},
  journal={Advances in Neural Information Processing Systems},
  volume={31},
  year={2018}
}

@article{mackay1992bayesian,
  title={Bayesian interpolation},
  author={MacKay, David JC},
  journal={Neural Computation},
  volume={4},
  number={3},
  pages={415--447},
  year={1992},
  publisher={MIT Press One Rogers Street, Cambridge, MA 02142-1209, USA journals-info~…}
}

@inproceedings{immer2021scalable,
  title = {Scalable marginal likelihood estimation for model selection in deep learning},
  author = {Immer, Alexander and Bauer, Matthias and Fortuin, Vincent and R{\"a}tsch, Gunnar and Khan, Mohammad Emtiyaz},
  booktitle = {Proceedings of the 38th International Conference on Machine Learning},
  series = {Proceedings of Machine Learning Research},
  volume = {139},
  pages = {4563--4573},
  publisher = {PMLR},
  year = {2021},
  url = {https://proceedings.mlr.press/v139/immer21a.html}
}

@article{hodgkinson2023interpolating,
  title={The interpolating information criterion for overparameterized models},
  author={Hodgkinson, Liam and van der Heide, Chris and Salomone, Robert and Roosta, Fred and Mahoney, Michael W},
  journal={arXiv preprint arXiv:2307.07785},
  year={2023}
}

@article{tulino2004random,
  title = {Random matrix theory and wireless communications},
  author = {Tulino, Antonia M. and Verd{\'u}, Sergio},
  journal = {Foundations and Trends in Communications and Information Theory},
  volume = {1},
  number = {1},
  pages = {1--182},
  year = {2004},
  publisher = {Now Publishers}
}

@article{mei2022generalization,
  title={The generalization error of random features regression: Precise asymptotics and the double descent curve},
  author={Mei, Song and Montanari, Andrea},
  journal={Communications on Pure and Applied Mathematics},
  volume={75},
  number={4},
  pages={667--766},
  year={2022},
  publisher={Wiley Online Library}
}

@inproceedings{liao2020random,
  title={A random matrix analysis of random {Fourier} features: Beyond the {Gaussian} kernel, a precise phase transition, and the corresponding double descent},
  author={Liao, Zhenyu and Couillet, Romain and Mahoney, Michael W},
  booktitle={Advances in Neural Information Processing Systems},
  volume={33},
  pages={13939--13950},
  year={2020}
}

@article{LiaoMahoney2026,
  author  = {Zhenyu Liao and Michael W. Mahoney},
  title   = {Random Matrix Theory for Deep Learning: Beyond Eigenvalues of Linear Models},
  journal = {IEEE Signal Processing Magazine},
  year    = {2026},
  volume  = {43},
  number  = {2},
  pages   = {93--106}
}

@inproceedings{ameli2025determinant,
  title = {Determinant estimation under memory constraints and neural scaling laws},
  author = {Ameli, Siavash and van der Heide, Chris and Hodgkinson, Liam and Roosta, Fred and Mahoney, Michael W.},
  booktitle = {Proceedings of the 42nd International Conference on Machine Learning},
  series = {Proceedings of Machine Learning Research},
  volume = {267},
  pages = {1367--1400},
  publisher = {PMLR},
  year = {2025},
}

@inproceedings{hodgkinson2025models,
  title = {Models of heavy-tailed mechanistic universality},
  author = {Hodgkinson, Liam and Wang, Zhichao and Mahoney, Michael W.},
  booktitle = {Proceedings of the 42nd International Conference on Machine Learning},
  series = {Proceedings of Machine Learning Research},
  volume = {267},
  pages = {23290--23329},
  publisher = {PMLR},
  year = {2025},
  url = {https://proceedings.mlr.press/v267/hodgkinson25a.html}
}

@misc{DLMF,
  author = {{DLMF}},
  title = {{NIST Digital Library of Mathematical Functions}},
  howpublished = {\url{https://dlmf.nist.gov/}, Release 1.2.4 of 2025-03-15},
  note = {F. W. J. Olver, A. B. Olde Daalhuis, D. W. Lozier,
          B. I. Schneider, R. F. Boisvert, C. W. Clark,
          B. R. Miller, B. V. Saunders, H. S. Cohl,
          and M. A. McClain, eds.}
}

@article{martin2021predicting,
  title={Predicting trends in the quality of state-of-the-art neural networks without access to training or testing data},
  author={Martin, Charles H and Peng, Tongsu and Mahoney, Michael W},
  journal={Nature Communications},
  volume={12},
  number={1},
  pages={4122},
  year={2021},
  publisher={Nature Publishing Group UK London}
}

@article{martin2021implicit,
  title={Implicit self-regularization in deep neural networks: Evidence from random matrix theory and implications for learning},
  author={Martin, Charles H and Mahoney, Michael W},
  journal={Journal of Machine Learning Research},
  volume={22},
  number={165},
  pages={1--73},
  year={2021}
}

@INPROCEEDINGS{MM19_HTSR_ICML,
  author =       {C. H. Martin and M. W. Mahoney},
  title =        {Traditional and Heavy-Tailed Self Regularization in Neural Network Models},
  booktitle =    {Proceedings of the 36th International Conference on Machine Learning},
  year =         {2019},
  pages =        {4284--4293},
}

@inproceedings{martin2020heavy,
  title={Heavy-tailed universality predicts trends in test accuracies for very large pre-trained deep neural networks},
  author={Martin, Charles H and Mahoney, Michael W},
  booktitle={Proceedings of the 2020 SIAM International Conference on Data Mining},
  pages={505--513},
  year={2020},
  organization={SIAM}
}

@article{setolTR,
  title={{SETOL}: A semi-empirical theory of (deep) learning},
  author={Martin, Charles H and Hinrichs, Christopher},
  journal={arXiv preprint arXiv:2507.17912},
  year={2025}
}

@article{agrawal2022alpha,
  title={{$\alpha$-ReQ}: Assessing Representation Quality in Self-Supervised Learning by measuring eigenspectrum decay},
  author={Agrawal, Kumar K and Mondal, Arnab Kumar and Ghosh, Arna and Richards, Blake},
  journal={Advances in Neural Information Processing Systems},
  volume={35},
  pages={17626--17638},
  year={2022}
}

@article{bartlett2020benign,
  title={Benign overfitting in linear regression},
  author={Bartlett, Peter L and Long, Philip M and Lugosi, G{\'a}bor and Tsigler, Alexander},
  journal={Proceedings of the National Academy of Sciences},
  volume={117},
  number={48},
  pages={30063--30070},
  year={2020},
  publisher={National Academy of Sciences}
}

@article{golub1973differentiation,
  title={The differentiation of pseudo-inverses and nonlinear least squares problems whose variables separate},
  author={Golub, Gene H and Pereyra, Victor},
  journal={SIAM Journal on Numerical Analysis},
  volume={10},
  number={2},
  pages={413--432},
  year={1973},
  publisher={SIAM}
}

@article{ismail1988linear,
  title={Linear birth and death models and associated {Laguerre} and {Meixner} polynomials},
  author={Ismail, Mourad EH and Letessier, Jean and Valent, Galliano},
  journal={Journal of Approximation Theory},
  volume={55},
  number={3},
  pages={337--348},
  year={1988},
  publisher={Elsevier}
}

@inproceedings{lotfi2022bayesian,
  title={Bayesian model selection, the marginal likelihood, and generalization},
  author={Lotfi, Sanae and Izmailov, Pavel and Benton, Gregory and Goldblum, Micah and Wilson, Andrew Gordon},
  booktitle={International Conference on Machine Learning},
  pages={14223--14247},
  year={2022},
  organization={PMLR}
}

@inproceedings{simonyan2014very,
  title = {Very deep convolutional networks for large-scale image recognition},
  author = {Simonyan, Karen and Zisserman, Andrew},
  booktitle = {International Conference on Learning Representations},
  year = {2015},
}

@inproceedings{tan2021efficientnetv2,
  title={{EfficientNetV2}: Smaller models and faster training},
  author={Tan, Mingxing and Le, Quoc},
  booktitle={International Conference on Machine Learning},
  pages={10096--10106},
  year={2021},
  organization={PMLR}
}

@inproceedings{tan2019mnasnet,
  title={{MnasNet}: Platform-aware neural architecture search for mobile},
  author={Tan, Mingxing and Chen, Bo and Pang, Ruoming and Vasudevan, Vijay and Sandler, Mark and Howard, Andrew and Le, Quoc V},
  booktitle={Proceedings of the IEEE/CVF Conference on Computer Vision and Pattern Recognition},
  pages={2820--2828},
  year={2019}
}

@article{dherin2022neural,
  title={Why neural networks find simple solutions: The many regularizers of geometric complexity},
  author={Dherin, Benoit and Munn, Michael and Rosca, Mihaela and Barrett, David},
  journal={Advances in Neural Information Processing Systems},
  volume={35},
  pages={2333--2349},
  year={2022}
}

@article{tsilivis2026flavors,
  author  = {Nikolaos Tsilivis and Eitan Gronich
             and Julia Kempe and Gal Vardi},
  title   = {Flavors of margin: Implicit bias of steepest
             descent in homogeneous neural networks},
  journal = {Journal of Machine Learning Research},
  year    = {2026},
  volume  = {27},
  number  = {104},
  pages   = {1--37},
  url     = {https://jmlr.org/papers/v27/25-0634.html}
}

@inproceedings{yang2022evaluating,
  title = {Test accuracy vs. generalization gap: Model selection in {NLP} without accessing training or testing data},
  author = {Yang, Yaoqing and Theisen, Ryan and Hodgkinson, Liam and Gonzalez, Joseph E. and Ramchandran, Kannan and Martin, Charles H. and Mahoney, Michael W.},
  booktitle = {Proceedings of the 29th ACM SIGKDD Conference on Knowledge Discovery and Data Mining},
  pages = {3011--3021},
  publisher = {Association for Computing Machinery},
  year = {2023},
  doi = {10.1145/3580305.3599518}
}

@inproceedings{krizhevsky2012imagenet,
  title={{ImageNet} classification with deep convolutional neural networks},
  author={Krizhevsky, Alex and Sutskever, Ilya and Hinton, Geoffrey E},
  booktitle={Advances in Neural Information Processing Systems},
  volume={25},
  year={2012}
}

@inproceedings{zhang2016understanding,
  title = {Understanding deep learning requires rethinking generalization},
  author = {Zhang, Chiyuan and Bengio, Samy and Hardt, Moritz and Recht, Benjamin and Vinyals, Oriol},
  booktitle = {International Conference on Learning Representations},
  year = {2017},
  url = {https://openreview.net/forum?id=Sy8gdB9xx}
}

@article{zhang2021understanding,
  title={Understanding deep learning (still) requires rethinking generalization},
  author={Zhang, Chiyuan and Bengio, Samy and Hardt, Moritz and Recht, Benjamin and Vinyals, Oriol},
  journal={Communications of the ACM},
  volume={64},
  number={3},
  pages={107--115},
  year={2021},
  publisher={ACM New York, NY, USA}
}

@article{belkin2019reconciling,
  title={Reconciling modern machine-learning practice and the classical bias--variance trade-off},
  author={Belkin, Mikhail and Hsu, Daniel and Ma, Siyuan and Mandal, Soumik},
  journal={Proceedings of the National Academy of Sciences},
  volume={116},
  number={32},
  pages={15849--15854},
  year={2019},
  publisher={National Academy of Sciences}
}

@article{kaplan2020scaling,
  title={Scaling laws for neural language models},
  author={Kaplan, Jared and McCandlish, Sam and Henighan, Tom and Brown, Tom B and Chess, Benjamin and Child, Rewon and Gray, Scott and Radford, Alec and Wu, Jeffrey and Amodei, Dario},
  journal={arXiv preprint arXiv:2001.08361},
  year={2020}
}

@inproceedings{hoffmann2022training,
  title = {An empirical analysis of compute-optimal large language model training},
  author = {Hoffmann, Jordan and Borgeaud, Sebastian and Mensch, Arthur and Buchatskaya, Elena and Cai, Trevor and Rutherford, Eliza and de Las Casas, Diego and Hendricks, Lisa Anne and Welbl, Johannes and Clark, Aidan and Hennigan, Thomas and Noland, Eric and Millican, Katherine and van den Driessche, George and Damoc, Bogdan and Guy, Aurelia and Osindero, Simon and Simonyan, Kar{\'e}n and Elsen, Erich and Vinyals, Oriol and Rae, Jack and Sifre, Laurent},
  booktitle = {Advances in Neural Information Processing Systems},
  volume = {35},
  pages = {30016--30030},
  year = {2022},
  url = {https://proceedings.neurips.cc/paper_files/paper/2022/hash/c1e2faff6f588870935f114ebe04a3e5-Abstract.html}
}

@article{boopathy2024unified,
  title={Unified neural network scaling laws and scale-time equivalence},
  author={Boopathy, Akhilan and Fiete, Ila},
  journal={arXiv preprint arXiv:2409.05782},
  year={2024}
}

@inproceedings{mcallester2003simplified,
  title={Simplified {PAC-Bayesian} margin bounds},
  author={McAllester, David},
  booktitle={Learning Theory and Kernel Machines: 16th Annual Conference on Learning Theory and 7th Kernel Workshop, COLT/Kernel 2003, Washington, DC, USA, August 24-27, 2003. Proceedings},
  pages={203--215},
  year={2003},
  organization={Springer}
}

@article{lotfi2022pac,
  title={{PAC-Bayes} compression bounds so tight that they can explain generalization},
  author={Lotfi, Sanae and Finzi, Marc and Kapoor, Sanyam and Potapczynski, Andres and Goldblum, Micah and Wilson, Andrew G},
  journal={Advances in Neural Information Processing Systems},
  volume={35},
  pages={31459--31473},
  year={2022}
}

@article{livni2020limitation,
  title={A limitation of the {PAC-Bayes} framework},
  author={Livni, Roi and Moran, Shay},
  journal={Advances in Neural Information Processing Systems},
  volume={33},
  pages={20543--20553},
  year={2020}
}

@book{konishi2008information,
  title={Information criteria and statistical modeling},
  author={Konishi, Sadanori and Kitagawa, Genshiro},
  year={2008},
  publisher={Springer}
}

@book{watanabe2009algebraic,
  title={Algebraic geometry and statistical learning theory},
  author={Watanabe, Sumio},
  volume={25},
  year={2009},
  publisher={Cambridge University Press}
}

@inproceedings{keskar2016large,
  title = {On large-batch training for deep learning: Generalization gap and sharp minima},
  author = {Keskar, Nitish Shirish and Mudigere, Dheevatsa and Nocedal, Jorge and Smelyanskiy, Mikhail and Tang, Ping Tak Peter},
  booktitle = {International Conference on Learning Representations},
  year = {2017},
  url = {https://openreview.net/forum?id=H1oyRlYgg}
}

@inproceedings{dinh2017sharp,
  title={Sharp minima can generalize for deep nets},
  author={Dinh, Laurent and Pascanu, Razvan and Bengio, Samy and Bengio, Yoshua},
  booktitle={International Conference on Machine Learning},
  pages={1019--1028},
  year={2017},
  organization={PMLR}
}

@article{qi2009sharp,
  title = {Sharp inequalities for polygamma functions},
  author = {Guo, Bai-Ni and Qi, Feng and Zhao, Jiao-Lian and Luo, Qiu-Ming},
  journal = {Mathematica Slovaca},
  volume = {65},
  number = {1},
  pages = {103--120},
  year = {2015},
  doi = {10.1515/ms-2015-0010}
}

@article{dosovitskiy2020image,
  title={An image is worth 16x16 words: Transformers for image recognition at scale},
  author={Dosovitskiy, Alexey and Beyer, Lucas and Kolesnikov, Alexander and Weissenborn, Dirk and Zhai, Xiaohua and Unterthiner, Thomas and Dehghani, Mostafa and Minderer, Matthias and Heigold, Georg and Gelly, Sylvain and others},
  journal={International Conference on Learning Representations},
  year={2021}
}

@inproceedings{liu2021swin,
  title = {{Swin Transformer}: Hierarchical vision transformer using shifted windows},
  author = {Liu, Ze and Lin, Yutong and Cao, Yue and Hu, Han and Wei, Yixuan and Zhang, Zheng and Lin, Stephen and Guo, Baining},
  booktitle = {Proceedings of the IEEE/CVF International Conference on Computer Vision},
  pages = {10012--10022},
  year = {2021}
}

@inproceedings{liu2022swin,
  title={{Swin Transformer V2}: Scaling up capacity and resolution},
  author={Liu, Ze and Hu, Han and Lin, Yutong and Yao, Zhuliang and Xie, Zhenda and Wei, Yixuan and Ning, Jia and Cao, Yue and Zhang, Zheng and Dong, Li and others},
  booktitle={Proceedings of the IEEE/CVF Conference on Computer Vision and Pattern Recognition},
  pages={12009--12019},
  year={2022}
}

@article{singh2025openai,
  title={{OpenAI GPT-5} system card},
  author={Singh, Aaditya and Fry, Adam and Perelman, Adam and Tart, Adam and Ganesh, Adi and El-Kishky, Ahmed and McLaughlin, Aidan and Low, Aiden and Ostrow, AJ and Ananthram, Akhila and others},
  journal={arXiv preprint arXiv:2601.03267},
  year={2025}
}

@inproceedings{zhou2023temperature,
  title = {Temperature balancing, layer-wise weight analysis, and neural network training},
  author = {Zhou, Yefan and Pang, Tianyu and Liu, Keqin and Martin, Charles H. and Mahoney, Michael W. and Yang, Yaoqing},
  booktitle = {Advances in Neural Information Processing Systems},
  volume = {36},
  pages = {63542--63572},
  year = {2023},
  url = {https://proceedings.neurips.cc/paper_files/paper/2023/hash/c8a4dd7d9e13583d714ce8580da7bbc7-Abstract-Conference.html}
}

@inproceedings{lu2024alphapruning,
  title={{AlphaPruning}: Using heavy-tailed self regularization theory for improved layer-wise pruning of large language models},
  author={Lu, Haiquan and Zhou, Yefan and Liu, Shiwei and Wang, Zhangyang and Mahoney, Michael W and Yang, Yaoqing},
  booktitle={Advances in Neural Information Processing Systems},
  volume={37},
  pages={9117--9152},
  year={2024}
}

@article{hu2025eigenspectrum,
  title={Eigenspectrum analysis of neural networks without aspect ratio bias},
  author={Hu, Yuanzhe and Goel, Kinshuk and Killiakov, Vlad and Yang, Yaoqing},
  journal={International Conference on Machine Learning},
  year={2025}
}

@inproceedings{munn2024bayesian,
  title = {A {Bayesian} model selection criterion for selecting pretraining checkpoints},
  author = {Munn, Michael and Wei, Susan},
  booktitle = {Proceedings of the 42nd International Conference on Machine Learning},
  series = {Proceedings of Machine Learning Research},
  volume = {267},
  pages = {45256--45271},
  publisher = {PMLR},
  year = {2025},
  url = {https://proceedings.mlr.press/v267/munn25a.html}
}

@inproceedings{lau2023local,
  title = {The local learning coefficient: A singularity-aware complexity measure},
  author = {Lau, Edmund and Furman, Zach and Wang, George and Murfet, Daniel and Wei, Susan},
  booktitle = {Proceedings of the 28th International Conference on Artificial Intelligence and Statistics},
  series = {Proceedings of Machine Learning Research},
  volume = {258},
  pages = {244--252},
  publisher = {PMLR},
  year = {2025},
  url = {https://proceedings.mlr.press/v258/lau25a.html}
}

@inproceedings{wei2022more,
  title={More than a toy: Random matrix models predict how real-world neural representations generalize},
  author={Wei, Alexander and Hu, Wei and Steinhardt, Jacob},
  booktitle={International Conference on Machine Learning},
  pages={23549--23588},
  year={2022},
  organization={PMLR}
}

@inproceedings{wang2024nonlinear,
  title={Nonlinear spiked covariance matrices and signal propagation in deep neural networks},
  author={Wang, Zhichao and Wu, Denny and Fan, Zhou},
  booktitle={The Thirty Seventh Annual Conference on Learning Theory},
  pages={4891--4957},
  year={2024},
  organization={PMLR}
}

@article{atanasov2025two,
  title = {Two-point deterministic equivalence for stochastic gradient dynamics in linear models},
  author = {Atanasov, Alexander and Bordelon, Blake and Zavatone-Veth, Jacob A. and Paquette, Courtney and Pehlevan, Cengiz},
  journal = {Advances in Theoretical and Mathematical Physics},
  volume = {30},
  number = {1},
  pages = {3--39},
  year = {2026},
  doi = {10.4310/ATMP.260412235902}
}

@article{ba2022high,
  title={High-dimensional asymptotics of feature learning: How one gradient step improves the representation},
  author={Ba, Jimmy and Erdogdu, Murat A and Suzuki, Taiji and Wang, Zhichao and Wu, Denny and Yang, Greg},
  journal={Advances in Neural Information Processing Systems},
  volume={35},
  pages={37932--37946},
  year={2022}
}

@inproceedings{chaudhari2018stochastic,
  title={Stochastic gradient descent performs variational inference, converges to limit cycles for deep networks},
  author={Chaudhari, Pratik and Soatto, Stefano},
  booktitle={2018 Information Theory and Applications Workshop (ITA)},
  pages={1--10},
  year={2018},
  organization={IEEE}
}

@article{xie2022power,
  title={On the power-law {Hessian} spectrums in deep learning},
  author={Xie, Zeke and Tang, Qian-Yuan and Cai, Yunfeng and Sun, Mingming and Li, Ping},
  journal={arXiv preprint arXiv:2201.13011},
  year={2022}
}

@article{sonthalia2024regularization,
  title={On regularization via early stopping for least squares regression},
  author={Sonthalia, Rishi and Lok, Jackie and Rebrova, Elizaveta},
  journal={arXiv preprint arXiv:2406.04425},
  year={2024}
}

@article{hodgkinson2023pac,
  title={A {PAC-Bayesian} perspective on the interpolating information criterion},
  author={Hodgkinson, Liam and van der Heide, Chris and Salomone, Robert and Roosta, Fred and Mahoney, Michael W},
  journal={NeurIPS 2023 Workshop on Mathematics of Modern Machine Learning},
  year={2023}
}

@article{martin2017rethinking,
  title={Rethinking generalization requires revisiting old ideas: statistical mechanics approaches and complex learning behavior},
  author={Martin, Charles H and Mahoney, Michael W},
  journal={arXiv preprint arXiv:1710.09553},
  year={2017}
}

@inproceedings{AMELI-2025b,
    title={Spectral estimation with free decompression},
    author={Siavash Ameli and Chris van der Heide and Liam Hodgkinson and Michael W. Mahoney},
    booktitle={The Thirty-ninth Annual Conference on Neural Information Processing Systems},
    year={2025},
}

@article{gonzalez2015laurent,
  title={Laurent expansion of the inverse of perturbed, singular matrices},
  author={Gonzalez-Rodriguez, Pedro and Moscoso, Miguel and Kindelan, Manuel},
  journal={Journal of Computational Physics},
  volume={299},
  pages={307--319},
  year={2015},
  publisher={Elsevier}
}

@article{srivastava2014dropout,
  title={Dropout: a simple way to prevent neural networks from overfitting},
  author={Srivastava, Nitish and Hinton, Geoffrey and Krizhevsky, Alex and Sutskever, Ilya and Salakhutdinov, Ruslan},
  journal={The Journal of Machine Learning Research},
  volume={15},
  number={1},
  pages={1929--1958},
  year={2014},
  publisher={JMLR. org}
}

@misc{WeightWatcher,
  author = {Martin, Charles},
  title = {{WeightWatcher}},
  year = {2024},
  publisher = {GitHub},
  journal = {GitHub repository},
  howpublished = {\url{https://github.com/CalculatedContent/WeightWatcher}}
}

@inproceedings{biderman2023pythia,
  title={Pythia: A suite for analyzing large language models across training and scaling},
  author={Biderman, Stella and Schoelkopf, Hailey and Anthony, Quentin Gregory and Bradley, Herbie and O’Brien, Kyle and Hallahan, Eric and Khan, Mohammad Aflah and Purohit, Shivanshu and Prashanth, USVSN Sai and Raff, Edward and others},
  booktitle={International Conference on Machine Learning},
  pages={2397--2430},
  year={2023},
  organization={PMLR}
}

@INPROCEEDINGS{MO11-implementing, 
  author =       {M.~W.~Mahoney and L.~Orecchia},
  title =        {Implementing regularization implicitly via approximate eigenvector computation},
  booktitle =    {Proceedings of the 28th International Conference on Machine Learning},
  year =         {2011},
  pages =        {121--128},
}

@INPROCEEDINGS{Mah12,
  author =       {M.~W.~Mahoney},
  title =        {Approximate Computation and Implicit Regularization for Very Large-scale Data Analysis},
  booktitle =    {Proceedings of the 31st ACM Symposium on Principles of Database Systems},
  year =         {2012},
  pages =        {143--154},
}

@inproceedings{YGKM19_pyhessian_TR,
  title={{PyHessian}: Neural networks through the lens of the {Hessian}},
  author={Yao, Zhewei and Gholami, Amir and Keutzer, Kurt and Mahoney, Michael W},
  booktitle={2020 IEEE International Conference on Big Data (Big Data)},
  pages={581--590},
  year={2020},
  organization={IEEE}
}

@inproceedings{YGSKM20_adahessian_TR,
  title={{AdaHessian}: An adaptive second order optimizer for machine learning},
  author={Yao, Zhewei and Gholami, Amir and Shen, Sheng and Mustafa, Mustafa and Keutzer, Kurt and Mahoney, Michael},
  booktitle={Proceedings of the AAAI Conference on Artificial Intelligence},
  volume={35},
  pages={10665--10673},
  year={2021}
}

@article{watanabe2013widely,
  title={A widely applicable {Bayesian} information criterion},
  author={Watanabe, Sumio},
  journal={The Journal of Machine Learning Research},
  volume={14},
  number={1},
  pages={867--897},
  year={2013},
  publisher={JMLR. org}
}

@book{jost2017riemannian,
  title     = {Riemannian Geometry and Geometric Analysis},
  author    = {Jost, J{\"u}rgen},
  edition   = {7},
  publisher = {Springer},
  year      = {2017},
  doi       = {10.1007/978-3-319-61860-9}
}

@book{gupta1999matrix,
  title={Matrix variate distributions},
  author={Gupta, Arjun K and Nagar, Daya K},
  volume={104},
  year={1999},
  publisher={CRC Press}
}

@article{wang2021benign,
  title={Benign overfitting in multiclass classification: All roads lead to interpolation},
  author={Wang, Ke and Muthukumar, Vidya and Thrampoulidis, Christos},
  journal={Advances in Neural Information Processing Systems},
  volume={34},
  pages={24164--24179},
  year={2021}
}

@article{papyan2020prevalence,
  title={Prevalence of neural collapse during the terminal phase of deep learning training},
  author={Papyan, Vardan and Han, Xiao Y and Donoho, David L},
  journal={Proceedings of the National Academy of Sciences},
  volume={117},
  number={40},
  pages={24652--24663},
  year={2020},
  publisher={National Academy of Sciences}
}

@book{mingo2017free,
  title={Free probability and random matrices},
  author={Mingo, James A and Speicher, Roland},
  volume={35},
  year={2017},
  publisher={Springer}
}

@inproceedings{wilson2026last,
  title = {Is the last layer sufficient for uncertainty quantification?},
  author = {Wilson, Joseph and van der Heide, Chris and Hodgkinson, Liam and Roosta, Fred},
  booktitle = {Proceedings of the 43rd International Conference on Machine Learning},
  series = {Proceedings of Machine Learning Research},
  volume = {306},
  pages = {135094--135132},
  publisher = {PMLR},
  year = {2026},
}

@inproceedings{paperno2016lambada,
  title={The {LAMBADA} dataset: Word prediction requiring a broad discourse context},
  author={Paperno, Denis and Kruszewski, Germ{\'a}n and Lazaridou, Angeliki and Pham, Ngoc-Quan and Bernardi, Raffaella and Pezzelle, Sandro and Baroni, Marco and Boleda, Gemma and Fern{\'a}ndez, Raquel},
  booktitle={Proceedings of the 54th Annual Meeting of the Association for Computational Linguistics (Volume 1: Long Papers)},
  pages={1525--1534},
  year={2016}
}

@misc{qwen3.5,
  title = {{Qwen3.5}: Towards Native Multimodal Agents},
  author = {{Qwen Team}},
  month = {February},
  year = {2026},
  url = {https://qwen.ai/blog?id=qwen3.5}
}

@inproceedings{wang2024mmlu,
  title={{MMLU-Pro}: A more robust and challenging multi-task language understanding benchmark},
  author={Wang, Yubo and Ma, Xueguang and Zhang, Ge and Ni, Yuansheng and Chandra, Abhranil and Guo, Shiguang and Ren, Weiming and Arulraj, Aaran and He, Xuan and Jiang, Ziyan and others},
  booktitle={Advances in Neural Information Processing Systems},
  volume={37},
  pages={95266--95290},
  year={2024}
}

@article{bai1999methodologies,
  title = {Methodologies in spectral analysis of large dimensional random matrices, a review},
  author = {Bai, Zhidong},
  journal = {Statistica Sinica},
  volume = {9},
  number = {3},
  pages = {611--677},
  year = {1999},
  url = {https://www3.stat.sinica.edu.tw/statistica/j9n3/j9n31/j9n31.htm}
}
